\documentclass{article}

\newif\ifarxiv
\ifdefined\arxivmode
  \csname arxiv\arxivmode\endcsname
\else
  \arxivtrue
\fi

\usepackage{iclr2027_conference,times}

\ifarxiv
  \iclrfinalcopy
\fi

\usepackage{amsmath,amsfonts,bm}

\def\eqref#1{equation~\ref{#1}}

\def\1{\bm{1}}

\DeclareMathAlphabet{\mathsfit}{\encodingdefault}{\sfdefault}{m}{sl}
\SetMathAlphabet{\mathsfit}{bold}{\encodingdefault}{\sfdefault}{bx}{n}

\usepackage{hyperref}
\usepackage{url}
\usepackage{booktabs}
\usepackage{graphicx}
\usepackage{xcolor}
\usepackage{amsmath,amssymb,amsfonts,amsthm}
\usepackage{multirow}
\usepackage{xspace}
\usepackage{subcaption}
\usepackage{wrapfig}
\usepackage{placeins}

\ifarxiv
  \usepackage{authblk}

\fi

\newtheorem{theorem}{Theorem}

\renewcommand{\eqref}[1]{(\ref{#1})}

\graphicspath{{images/}}

\newcommand{\method}{MeRoPE\xspace}

\title{\method: Metric Rotary Position Embedding \\ for Camera-Controlled Video Generation}

\ifarxiv
  \author[1]{Zhijian Qiao}
  \author[2]{Xinjiang Wang}
  \author[2]{Jiajie Chen}
  \author[2]{Haoming Huang}
  \author[2]{Meng Li}
  \author[2]{\textbf{Chih-Chung Chou}}
  \author[2]{\textbf{Jing Wang}}
  \author[1,$\dagger$]{\textbf{Shaojie Shen}}
  \affil[1]{The Hong Kong University of Science and Technology}
  \affil[2]{Zhuoyu Technology}
  \affil[$\dagger$]{Corresponding author: \texttt{eeshaojie@ust.hk}}
\else
  \author{Anonymous Authors}
\fi

\begin{document}

\maketitle

\ifarxiv
  \begin{center}
    \vspace{-0.22in}
    \small Project page: \href{https://qiaozhijian.github.io/merope/}{\texttt{qiaozhijian.github.io/merope}}
  \end{center}
  \thispagestyle{fancy}
  \pagestyle{fancy}
  \lhead{Preprint.}
\fi

\begin{abstract}
In camera-controlled video generation, geometry-aware positional encodings condition tokens on camera extrinsics and per-token viewing rays.  Existing schemes, however, have a scale-dependent failure mode on real-world metric camera trajectories: homogeneous projective encodings cause attention logits and feature norms to grow unbounded with physical translation baselines.  We propose \method (\textbf{Me}tric \textbf{Ro}tary \textbf{P}osition \textbf{E}mbedding), a norm-preserving relative camera encoding for attention.  \method encodes relative orientations between calibrated viewing rays with orthogonal rotation blocks, maps raw metric displacements into multi-frequency rotary phases, and adds a disparity-anchored correspondence prior along the epipolar arc.  This design strictly preserves feature norms, bounds pre-softmax attention logits regardless of the physical translation scale, and maintains exact invariance to global rigid coordinate changes.  Across nuScenes and PanShot, which cover large-baseline trajectories and diverse camera optics, respectively, \method achieves stronger camera control than prior encodings, with the best consistency between generated camera motion and conditioning poses in both rotation and translation.  Code will be made publicly available.
\end{abstract}

\section{Introduction}
\label{sec:intro}

Geometry-aware positional encodings (PEs) enable camera-controlled video generation by making attention sensitive to camera poses and token-level viewing geometry.  However, representing real-world \emph{metric} translation within these encodings remains challenging: raw physical displacements can destabilize attention, whereas normalization discards physical scale.

Existing relative camera encodings based on homogeneous projective matrices~\citep{miyato2024gta,li2025prope,zhang2026ucpe} jointly encode rotation and translation in non-orthogonal attention operators.  In these formulations, physical translation enters as an unnormalized inner product, causing attention logits and feature norms to grow unbounded with the baseline length.  Over large displacements, translational terms artificially dominate attention scores and suppress visual content, degrading trajectory control.

This tension yields three task-driven desiderata for camera positional encodings: (A)~depth-independent full metric relative-pose dependence on $T_i^{-1}T_j$, which retains both motion direction and physical translation scale without coupling translation to an assumed scene depth; (B)~strict per-token factorization, which permits independent query/key transformations compatible with standard attention; and (C)~norm-preserving actions, which prevent baseline-dependent amplification of logits and features.  However, Theorem~\ref{thm:translation-blindness} shows that no continuous finite-dimensional encoding can satisfy all three while remaining sensitive to metric translation.  Because camera control requires both metric fidelity (A) and scale stability (C), we prioritize (A+C) and relax (B) through query-camera grouping.

We therefore introduce \method (\textbf{Me}tric \textbf{Ro}tary \textbf{P}osition \textbf{E}mbedding), a norm-preserving encoding with full metric relative-pose dependence.  \method constructs a minimum-rotation (MinRot) local frame for each calibrated 3D ray, represents relative ray orientation with orthogonal blocks, and maps query-frame metric displacements into multi-frequency rotary phases.  Together, these operators retain raw metric camera motion while bounding attention logits across arbitrary physical baselines.

Because this pose operator is scene-independent, it does not identify potential static-scene correspondences across views.  \method therefore encodes correspondence hypotheses as disparity-anchored rotations along spherical epipolar arcs.  Unlike CamCo~\citep{xu2024camco}, this design uses epipolar geometry to express relative cross-view relations without restricting which token pairs can interact.  Its angular disparity anchors span a bounded spherical arc, avoiding the dataset-specific metric-depth intervals of URoPE~\citep{xie2026urope} and the learned depth estimation of RayRoPE~\citep{wu2026rayrope}.

\paragraph{Contributions.}
\begin{itemize}
    \item We identify baseline-dependent logit and feature-norm growth in homogeneous projective camera encodings and formalize a trilemma among full metric relative-pose dependence, strict per-token factorization, and norm preservation.
    \item We propose \method, which resolves this trade-off through query-camera grouping and combines a norm-preserving metric-pose operator with disparity-anchored spherical rotations for cross-view relations.
    \item We demonstrate improved camera control on large-baseline nuScenes and camera-diverse PanShot, with \method achieving the best joint rotation--translation consistency on both benchmarks.
\end{itemize}

\section{Related Work}
\label{sec:related}

\paragraph{Camera conditioning pathways in video generation.}
Camera trajectories control viewpoint in video synthesis for applications such as autonomous driving and robotic simulation.  Existing approaches implement this control through several broad pathways.  \emph{Feature fusion and adapters} integrate compact pose or ray representations into generative backbones: MotionCtrl~\citep{wang2024motionctrl} concatenates camera extrinsics with temporal features, CameraCtrl~\citep{he2025cameractrl} adds multi-scale features encoded from per-pixel Pl\"ucker rays, ReCamMaster~\citep{bai2025recammaster} additively embeds target extrinsics into transformer features, and CamCo~\citep{xu2024camco} combines Pl\"ucker-conditioned feature fusion with epipolar attention.  These mechanisms provide lightweight control but do not by themselves encode relative viewpoint relations in attention.  \emph{Explicit 3D guidance and rendering} instead derives structural conditions from estimated geometry: Gen3C~\citep{ren2025gen3c} renders point-cloud caches, Lyra~2.0~\citep{shen2026lyra2} combines depth-based warping with geometry-aware memory, and ViewCrafter~\citep{yu2024viewcrafter} and TrajectoryCrafter~\citep{yu2025trajectorycrafter} condition generation on point-cloud renderings.  They provide direct spatial guidance, but errors in estimated depth, pose, or scene motion can produce holes, distortions, or misaligned conditions.  \emph{Adaptive modulation} injects camera-dependent affine transformations rather than directly fusing conditions: BulletTime~\citep{wang2025bullettime} combines Camera-AdaLN with camera-aware positional encoding, while CompoSIA~\citep{zhan2026composing} couples AdaLN-style residual ego-motion modulation with global projective attention.  \emph{Geometry-aware positional encodings} incorporate camera relations directly into attention without constructing an external 3D proxy at inference time.

\paragraph{Evolution of geometry-aware positional encodings.}
Within this family, early formulations operated on camera-level extrinsic transformations~\citep{kong2024eschernet,miyato2024gta} and projective viewing frustums~\citep{li2025prope}.  Specifically, CaPE~\citep{kong2024eschernet} applies block-diagonal relative pose matrices to queries and keys in multi-view attention, GTA~\citep{miyato2024gta} generalizes this to group representations of $SE(3)\times SO(2)\times SO(2)$ acting on queries, keys, and values, and PRoPE~\citep{li2025prope} incorporates intrinsics to model projective viewing frustums.  Subsequent methods construct per-token 3D ray representations: UCPE~\citep{zhang2026ucpe} builds token-level ray-local rigid frames from calibrated pixel-to-ray mappings to support diverse camera optics, and RayNova~\citep{xie2026raynova} applies multi-axis RoPE to Pl\"ucker coordinates in a fixed world frame.  Recent efforts further incorporate depth cues: URoPE~\citep{xie2026urope} lifts tokens to discrete depth anchors and reprojects them onto the query pixel plane, RayRoPE~\citep{wu2026rayrope} computes expected rotary phases from predicted depth segments, and CRePE~\citep{jin2026crepe} models curved rays with pseudo-depth expectation distributions.

\begin{wraptable}{r}{0.48\textwidth}
    \vspace{-3mm}
    \centering
    \captionsetup{justification=raggedright,singlelinecheck=false}
    \caption{Taxonomy under three properties: (A)~depth-independent full metric relative pose, (B)~strict per-token factorization, and (C)~norm-preserving attention action.}
    \label{tab:related-taxonomy}
    \small
    \setlength{\tabcolsep}{3pt}
    \begin{tabular}{lccc}
    \toprule
    Method & (A) & (B) & (C) \\
    \midrule
    CaPE~\citep{kong2024eschernet} & \textbf{Yes} & \textbf{Yes} & No \\
    GTA~\citep{miyato2024gta} & \textbf{Yes} & \textbf{Yes} & No \\
    PRoPE~\citep{li2025prope} & \textbf{Yes} & \textbf{Yes} & No \\
    UCPE~\citep{zhang2026ucpe} & \textbf{Yes} & \textbf{Yes} & No \\
    \midrule
    RayNova~\citep{xie2026raynova} & No & \textbf{Yes} & \textbf{Yes} \\
    ViewRope~\citep{xiang2026viewrope} & No & \textbf{Yes} & \textbf{Yes} \\
    \midrule
    URoPE~\citep{xie2026urope} & No & No & \textbf{Yes} \\
    RayRoPE~\citep{wu2026rayrope} & No & No & No \\
    CRePE~\citep{jin2026crepe} & No & No & No \\
    \midrule
    \textbf{\method} (Ours) & \textbf{Yes} & No & \textbf{Yes} \\
    \bottomrule
    \end{tabular}
    \vspace{-2mm}
\end{wraptable}

\paragraph{Taxonomy under the trilemma and MeRoPE.}
Table~\ref{tab:related-taxonomy} organizes camera positional encodings by three properties: (A)~depth-independent full metric relative-pose dependence on $T_i^{-1}T_j$, (B)~strict per-token factorization, and (C)~norm-preserving attention actions.
Homogeneous and projective encodings (CaPE, GTA, PRoPE, UCPE) preserve (A+B) but violate (C), as non-orthogonal translation blocks cause attention logits to scale unbounded with physical baselines.  World-frame ray models such as RayNova and direction-only encodings such as ViewRope~\citep{xiang2026viewrope} preserve (B+C) but sacrifice (A): RayNova depends on the chosen global coordinate frame, whereas ViewRope omits metric translation.  Query-frame reprojection approaches~\citep{xie2026urope,wu2026rayrope,jin2026crepe} relax (B) and do not satisfy (A), because their attention relations are defined in pixel coordinates.  Under central projection, jointly scaling camera translation and scene depth leaves projected coordinates unchanged, so pixel-space relations identify translation only up to a depth scale---the same projective ambiguity underlying monocular scale ambiguity.  Fixed or predicted depths select a scale but do not make the pixel-space relation intrinsically metric.  URoPE preserves (C), while the expected phasors used by RayRoPE and CRePE are not strictly norm-preserving.
\method retains (A+C) by relaxing (B) via query-camera grouping; Section~\ref{sec:method} verifies (A+C), and Theorem~\ref{thm:translation-blindness} establishes why (B) must be relaxed.

\section{Method}
\label{sec:method}

This section presents \method (\textbf{Me}tric \textbf{Ro}tary \textbf{P}osition \textbf{E}mbedding), a norm-preserving relative camera positional encoding for video generation.
Section~\ref{sec:setup} formalizes geometry-aware attention and analyzes the scale-dependent instability of homogeneous relative-pose operators.
Section~\ref{sec:pose-encoding} encodes ray-frame rotation and raw metric translation with orthogonal blocks, while Section~\ref{sec:sray} introduces disparity-anchored spherical rotations for potential static-scene correspondences.

\subsection{Geometry-Aware Attention and Metric-Scale Limitation}
\label{sec:setup}

We formalize how camera positional encodings act within a geometry-aware attention layer.  Let $a=(i,p)$ denote patch $p$ from query camera $i$, and let $b=(j,q)$ denote patch $q$ from key camera $j$, with corresponding query, key, and value vectors $\mathbf{q}_a, \mathbf{k}_b, \mathbf{v}_b \in \mathbb{R}^d$.
Geometry-aware attention modulates query-key comparison and value aggregation through a pairwise transformation operator $\mathcal{U}_{ab} \in \mathbb{R}^{d \times d}$:
\begin{equation}
s_{ab} = \frac{\mathbf{q}_a^\top\mathcal{U}_{ab}\mathbf{k}_b}{\sqrt d}, \qquad
\alpha_{ab} = \frac{\exp(s_{ab})}{\sum_{b'}\exp(s_{ab'})}, \qquad
\mathbf{y}_a = \sum_b \alpha_{ab}\mathcal{U}_{ab}\mathbf{v}_b.
\label{eq:geometry-conditioned-attention}
\end{equation}
In GTA~\citep{miyato2024gta}, the camera component of $\mathcal{U}_{ab}$ is a direct sum $\bigoplus \mathcal{U}^{\mathrm{cam}}_{ij}$ of $4 \times 4$ homogeneous relative-pose blocks:
\begin{equation}
R_{ij} := R_i^\top R_j, \qquad
\Delta\mathbf{o}_{j\mid i} := R_i^\top(\mathbf{o}_j - \mathbf{o}_i), \qquad
\mathcal{U}^{\mathrm{cam}}_{ij} :=
\begin{bmatrix}
R_{ij} & \Delta\mathbf{o}_{j\mid i} \\
\mathbf{0}^\top & 1
\end{bmatrix},
\label{eq:gta-camera-block}
\end{equation}
where $R_i \in \mathrm{SO}(3)$ and $\mathbf{o}_i \in \mathbb{R}^3$ denote the rotation and optical center of camera $i$.
Decomposing the 4D feature segments of $\mathbf{q}_a$ and $\mathbf{k}_b$ into 3D vectors and scalar coordinates, $[\mathbf{u}^{q}; h_q]$ and $[\mathbf{u}^{k}; h_k]$, the resulting attention score contribution is:
\begin{equation}
\begin{bmatrix}\mathbf{u}^{q}\\h_q\end{bmatrix}^{\!\top}
\mathcal{U}^{\mathrm{cam}}_{ij}
\begin{bmatrix}\mathbf{u}^{k}\\h_k\end{bmatrix}
=
(\mathbf{u}^{q})^\top R_{ij}\mathbf{u}^{k}
+ h_k(\mathbf{u}^{q})^\top\Delta\mathbf{o}_{j\mid i}
+ h_qh_k .
\label{eq:homogeneous-score}
\end{equation}

Because the second term $h_k(\mathbf{u}^{q})^\top\Delta\mathbf{o}_{j\mid i}$ is an unnormalized linear inner product, it scales linearly with the physical translation baseline $\|\mathbf{o}_j - \mathbf{o}_i\|_2$.
In large-baseline settings, this translational term numerically dominates the attention logits, suppressing both rotational geometry and semantic visual content.
Similarly, because $\mathcal{U}^{\mathrm{cam}}_{ij}$ is non-orthogonal, the aligned value features also scale unbounded with the baseline:
\begin{equation}
\left\|
\mathcal{U}^{\mathrm{cam}}_{ij}
\begin{bmatrix}\mathbf{u}^{v}\\h_v\end{bmatrix}
\right\|_2^2
=
\left\|R_{ij}\mathbf{u}^{v} + h_v\Delta\mathbf{o}_{j\mid i}\right\|_2^2 + h_v^2 .
\label{eq:homogeneous-norm}
\end{equation}

\begin{wrapfigure}{r}{0.42\textwidth}
    \vspace{-4mm}
    \centering
    \includegraphics[width=\linewidth]{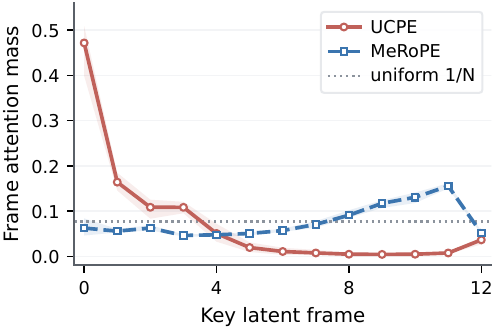}
    \vspace{-6mm}
    \caption{\textbf{Temporal attention under large baselines.}
    Cross-frame attention mass from the center query token of the final frame across all preceding key frames along a straight trajectory. As the baseline grows, UCPE (red) allocates more mass to distant early frames, whereas \method (blue) avoids this shift.}
    \label{fig:motivation-toy}
    \vspace{-11mm}
\end{wrapfigure}

Figure~\ref{fig:motivation-toy} empirically illustrates this limitation on a trained baseline model.
Evaluating temporal attention allocations along a forward-driving trajectory, we measure the cross-frame attention mass from the center query token of the final frame across all key frames.
UCPE assigns an increasing share of attention to distant early frames as the translation baseline grows, reducing the mass on adjacent frames despite their greater cross-view co-visibility.
\method bounds the positional contribution and avoids this baseline-driven shift.

\subsection{Orthogonal Metric Pose Encoding}
\label{sec:pose-encoding}

The intrinsic relative camera pose $(R_i^\top R_j, \Delta\mathbf{o}_{j\mid i})$ consists of relative rotation and metric displacement.
To prevent the norm and logit explosion of homogeneous matrices, we decompose relative pose into two decoupled orthogonal operators: a ray-local minimum-rotation (MinRot) block $\mathcal{U}^{\mathrm{rot}}_{ab}$ and a query-frame metric translation RoPE block $\mathcal{U}^{\mathrm{trans}}_{ij}$.

\paragraph{Ray-local coordinate framing via MinRot.}
Each image patch $p$ from camera $i$ is mapped to a calibrated unit viewing ray $\mathbf{d}_{i,p}^c \in \mathbb{S}^2$ in the camera coordinate frame, providing a camera-model-agnostic representation across heterogeneous optics.
To make attention sensitive to individual viewing directions rather than whole-camera poses alone, each token ray is assigned a local coordinate frame ($A_{i,p}$ and $A_{j,q}$ in Figure~\ref{fig:encoding-geometry}).
We construct this ray-local rotation matrix $A_{i,p} \in \mathrm{SO}(3)$ via the Rodrigues minimum-rotation formula, carrying the optical axis $\mathbf{e}_z$ directly onto $\mathbf{d}_{i,p}^c$:
\begin{equation}
A_{i,p} = I + [\mathbf{v}_{i,p}]_\times + \frac{[\mathbf{v}_{i,p}]_\times^2}{1 + c_{i,p}}, \quad \text{where } \mathbf{v}_{i,p} = \mathbf{e}_z \times \mathbf{d}_{i,p}^c, \; c_{i,p} = \mathbf{e}_z^\top \mathbf{d}_{i,p}^c.
\label{eq:minrot}
\end{equation}
Unlike UCPE~\citep{zhang2026ucpe}, whose cross-product frame degenerates when a token's vertical ray angle reaches $\pm 90^\circ$ (equivalently, as the vertical field of view approaches $180^\circ$), $A_{i,p}$ is singular only at the backward direction $\mathbf{d}_{i,p}^c = -\mathbf{e}_z$ ($-180^\circ$).
Transforming ray $p$ from camera $i$ into world coordinates yields the ray-local frame $R_{i,p}^{\mathrm{ray}} = R_i A_{i,p}$.
The pairwise relative rotation between query ray $(i,p)$ and key ray $(j,q)$ is then $A_{i,p}^\top R_i^\top R_j A_{j,q}$.
In attention, this relative rotation is embedded as a direct sum over $m$ repeated 3D rotation blocks:
\begin{equation}
\mathcal{U}^{\mathrm{rot}}_{ab} := \bigoplus_{k=1}^{m} \left( A_{i,p}^\top R_i^\top R_j A_{j,q} \right),
\label{eq:relative-ray-rotation}
\end{equation}
which strictly preserves feature norms while capturing the 3D relative orientation between rays.

\paragraph{Query-frame metric translation RoPE.}
To encode the query-frame metric displacement $\Delta\mathbf{o}_{j\mid i} = R_i^\top(\mathbf{o}_j - \mathbf{o}_i)$ illustrated in Figure~\ref{fig:encoding-geometry} without baseline-dependent logit growth, we map each translation coordinate into multi-frequency rotary phases.
Following RoPE frequency construction and position-interpolation principles~\citep{su2024roformer,chen2023extending}, given predefined minimum and maximum wavelengths $[\lambda_{\min}, \lambda_{\max}]$, we construct $K$ logarithmically spaced wavelengths and their corresponding angular frequencies:
\begin{equation}
\omega_k = \frac{2\pi}{\lambda_k}, \qquad \lambda_k = \lambda_{\min} \left(\frac{\lambda_{\max}}{\lambda_{\min}}\right)^{k/(K-1)}, \qquad k = 0, \dots, K-1.
\label{eq:translation-frequencies}
\end{equation}
For each Cartesian axis $c \in \{x,y,z\}$, the 1D displacement $[\Delta\mathbf{o}_{j\mid i}]_c$ rotates a 2D feature subspace by angle $\omega_k [\Delta\mathbf{o}_{j\mid i}]_c$.
The translation operator is then the direct sum over all 3 axes and $K$ frequencies:
\begin{equation}
\mathcal{U}^{\mathrm{trans}}_{ij} := \bigoplus_{c \in \{x,y,z\}} \bigoplus_{k=0}^{K-1} \operatorname{Rot}\big(\omega_k [\Delta\mathbf{o}_{j\mid i}]_c\big), \qquad \text{where } \operatorname{Rot}(\theta) = \begin{bmatrix} \cos\theta & -\sin\theta \\ \sin\theta & \cos\theta \end{bmatrix},
\label{eq:translation-rope}
\end{equation}
which is strictly orthogonal, preserving feature norms and preventing attention logits from scaling with the physical translation baseline.

\subsection{Disparity-Anchored Spherical Encoding}
\label{sec:sray}

\begin{wrapfigure}{r}{0.38\textwidth}
    \vspace{-15mm}
    \centering
    \includegraphics[width=\linewidth]{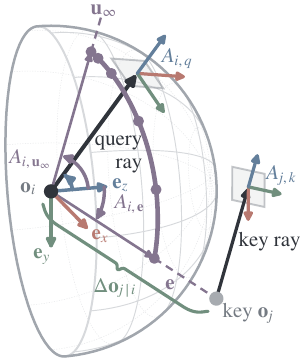}
    \vspace{-6mm}
    \caption{\textbf{Geometric construction.}
    Blue marks the query/key ray-local frames \textcolor[RGB]{99,127,155}{$A_{i,q}$} and \textcolor[RGB]{99,127,155}{$A_{j,k}$} used by rotation PE; green marks the query-frame metric displacement \textcolor[RGB]{113,143,126}{$\Delta\mathbf{o}_{j\mid i}$} used by translation PE; purple marks disparity PE, whose great-circle arc runs from the transformed key-ray direction \textcolor[RGB]{132,117,143}{$\mathbf{u}_\infty$} to the epipole \textcolor[RGB]{132,117,143}{$\mathbf{e}$}, with endpoint frames \textcolor[RGB]{132,117,143}{$A_{i,\mathbf{u}_\infty}$} and \textcolor[RGB]{132,117,143}{$A_{i,\mathbf{e}}$} and sampled anchor directions \textcolor[RGB]{132,117,143}{$\mathbf{u}_{i,\beta_\ell}$} in between.}
    \label{fig:encoding-geometry}
    \vspace{-11mm}
\end{wrapfigure}

Potential static-scene correspondences add geometric guidance beyond the scene-independent relative camera framing provided by the orthogonal metric pose operator.
Existing methods either restrict attention strictly to 2D pixel epipolar lines~\citep{xu2024camco} or project metric depth anchors back into the query image plane~\citep{xie2026urope}.
However, pixel-plane reprojections break the camera-model-agnostic ray interface, and unbounded metric depth ranges require heuristic, dataset-specific anchor choices.

We instead construct correspondence anchors in spherical ray space using bounded angular disparity fractions (Figure~\ref{fig:encoding-geometry}).
Let $\mathbf{u}_\infty = R_i^\top R_j \mathbf{d}_{j,q}^c$ be the key ray direction in the query camera frame; for a nonzero baseline, let $\mathbf{e} = \Delta\mathbf{o}_{j\mid i} / \|\Delta\mathbf{o}_{j\mid i}\|_2$ be the epipole unit vector.
As 3D scene depth varies from $\infty$ to $0$, points along key ray $(j,q)$ sweep a great-circle arc on the unit sphere connecting $\mathbf{u}_\infty$ and $\mathbf{e}$.
We sample $L$ discrete disparity anchors along this epipolar arc at angular fractions $\rho_\ell \in [0, 1]$:
\begin{equation}
\begin{gathered}
\mathbf{u}_{i, \beta_\ell} = \cos(\rho_\ell \beta_{\max})\,\mathbf{u}_\infty + \sin(\rho_\ell \beta_{\max})\,\hat{\mathbf{t}}, \\
\text{where } \hat{\mathbf{t}} = \frac{\mathbf{e} - (\mathbf{e}^\top \mathbf{u}_\infty)\mathbf{u}_\infty}{\|\mathbf{e} - (\mathbf{e}^\top \mathbf{u}_\infty)\mathbf{u}_\infty\|_2}, \quad \beta_{\max} = \arccos(\mathbf{e}^\top \mathbf{u}_\infty).
\end{gathered}
\label{eq:disparity-anchor}
\end{equation}
For each anchor $\ell$, the Rodrigues rotation $A_{i, \beta_\ell}$ carries the query-camera optical axis $\mathbf{e}_z$ onto the query-frame direction $\mathbf{u}_{i, \beta_\ell}$.
The disparity encoding computes the relative rotation $A_{i,p}^\top A_{i, \beta_\ell}$ from query ray $\mathbf{d}_{i,p}^c$ to candidate direction $\mathbf{u}_{i, \beta_\ell}$, repeated across $n$ coordinate triplets:
\begin{equation}
\mathcal{U}^{\mathrm{disp}}_{(i,p),(j,q)} := \bigoplus_{\ell=1}^{L} \underbrace{\bigoplus_{k=1}^{n} \left( A_{i,p}^\top A_{i, \beta_\ell} \right)}_{n \text{ times}}.
\label{eq:s2-rot}
\end{equation}
These anchors act as fixed geometric knots rather than predictions of mutually exclusive correspondences.  Together, their rotation blocks encode the query ray relative to several fixed directions along the continuous arc, while leaving the network to determine from visual content whether a valid 3D correspondence exists.

Degenerate configurations make the construction in Eq.~\eqref{eq:disparity-anchor} undefined: the epipole $\mathbf{e}$ is undefined when the camera centers coincide, while the great-circle direction $\hat{\mathbf{t}}$ is undefined when $\mathbf{e}$ is parallel or antiparallel to $\mathbf{u}_\infty$.  In either case, we collapse all anchor directions to $\mathbf{u}_\infty$ to avoid division by zero.

The complete pairwise \method operator $\mathcal{U}_{ab}$ is assembled by the block-diagonal direct sum:
\begin{equation}
\mathcal{U}_{ab} = \mathcal{U}^{\mathrm{disp}}_{(i,p),(j,q)} \oplus \mathcal{U}^{\mathrm{rot}}_{(i,p),(j,q)} \oplus \mathcal{U}^{\mathrm{trans}}_{ij} \oplus \mathcal{U}^{\mathrm{native}},
\label{eq:joint-operator-method}
\end{equation}
where $\mathcal{U}^{\mathrm{native}}$ retains the backbone's native spatio-temporal RoPE over the temporal and image-plane $(x,y)$ coordinate bands (Figure~\ref{fig:encoding-operator}).
Thus, \method satisfies (A+C): it uses the full metric $T_i^{-1}T_j$ without scale normalization, while $\mathcal{U}_{ab}^\top\mathcal{U}_{ab}=I$ ensures norm preservation and baseline-independent logits.

\begin{wrapfigure}{r}{0.42\textwidth}
    \vspace{-6mm}
    \centering
    \includegraphics[width=\linewidth]{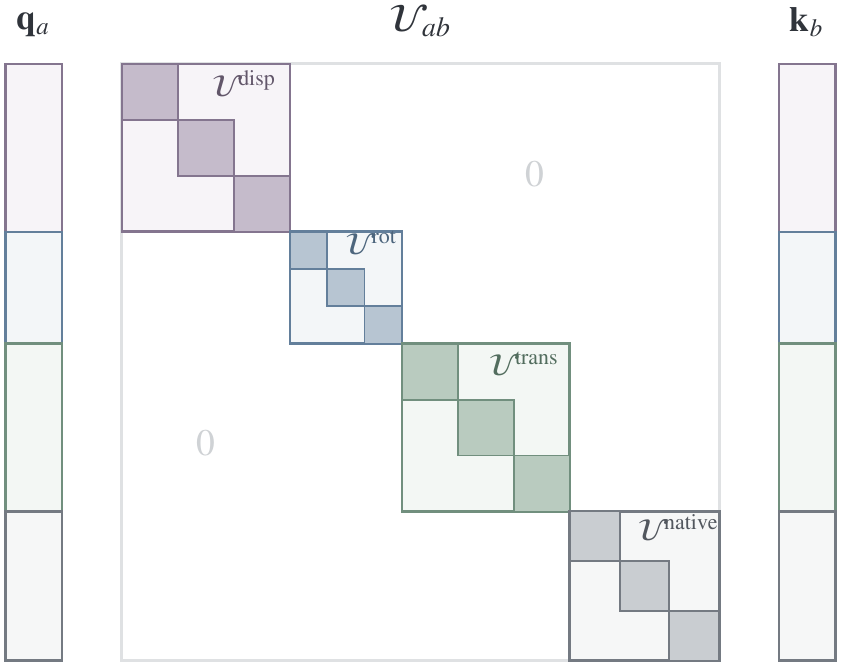}
    \vspace{-6mm}
    \caption{\textbf{Block-diagonal structure of $\mathcal{U}_{ab}$.}
    Pairwise attention query-key modulation $\mathbf{q}_a^\top \mathcal{U}_{ab} \mathbf{k}_b$, assembled as the orthogonal direct sum of disparity MinRot $\mathcal{U}^{\mathrm{disp}}$, ray-local MinRot $\mathcal{U}^{\mathrm{rot}}$, translation RoPE $\mathcal{U}^{\mathrm{trans}}$, and native RoPE $\mathcal{U}^{\mathrm{native}}$.}
    \label{fig:encoding-operator}
    \vspace{-14mm}
\end{wrapfigure}

\paragraph{Integration into Diffusion Transformers.}
Following the adapter scheme in UCPE~\citep{zhang2026ucpe}, we integrate \method into pretrained video Diffusion Transformers through a lightweight parallel attention branch.
The base model retains its original self-attention and spatio-temporal representations, while the parallel branch evaluates camera-conditioned attention modulated by $\mathcal{U}_{ab}$.
The resulting features are mapped back through a zero-initialized linear projection and added residually to the backbone stream, so the pretrained backbone features remain unchanged at initialization.

\section{Experiments}
\label{sec:experiments}

We evaluate \method on two benchmarks: large-baseline driving video generation on nuScenes~\citep{caesar2020nuscenes} (Section~\ref{sec:exp_video}) and multi-lens camera control on PanShot~\citep{zhang2026ucpe,zhang2026ucpecode} (Section~\ref{sec:exp_panshot}).

\subsection{Experimental Setup}
\label{sec:exp_setup}

\begin{wrapfigure}{r}{0.38\textwidth}
    \vspace{-16mm}
    \centering
    \includegraphics[width=\linewidth]{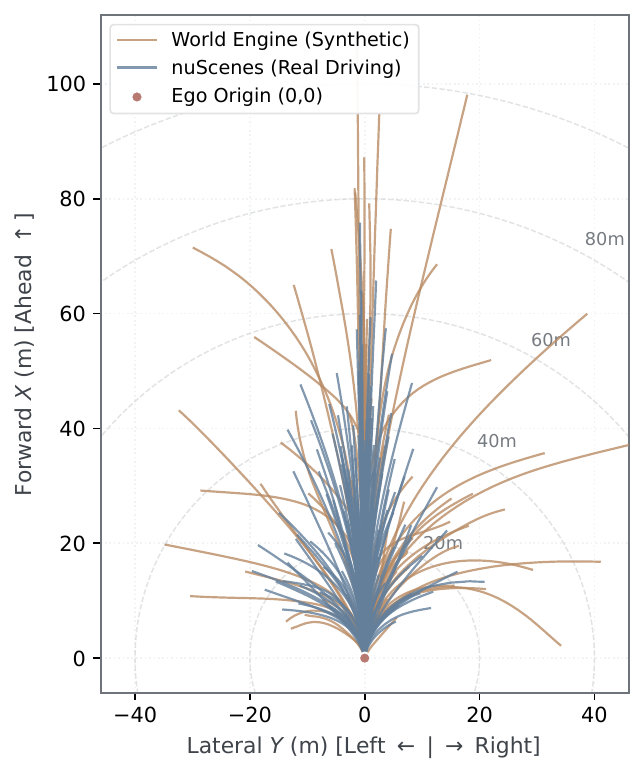}
    \vspace{-6mm}
    \caption{\textbf{Training trajectory distribution.}
    Bird's-eye view (BEV) of motion patterns in the training mixture: forward-dominant corridors from nuScenes (blue, $103.6\text{k}$) and wide lateral/turning maneuvers from World Engine (orange, $244.1\text{k}$).}
    \label{fig:train-trajectories}
    \vspace{-5mm}
\end{wrapfigure}

\paragraph{Datasets and training recipe.}
For driving video generation, training data combines real clips from the nuScenes training set ($103.6\text{k}$ clips) with synthetic trajectories rendered in 3D Gaussian environments by World Engine~\citep{li2026worldengine} ($244.1\text{k}$ clips), totaling $347.7\text{k}$ unique training sequences (Figure~\ref{fig:train-trajectories}).
Candidate trajectories are grouped using $k$-means and filtered by non-maximum suppression (NMS) with a Fr\'echet distance threshold of $1\,\text{m}$ to obtain representative motion prototypes. During training, class-balanced sampling draws mini-batches uniformly across trajectory clusters.
All models are trained at $288 \times 512$ resolution with $49$ frames, conditioned on the first $9$ RGB frames, commanded camera trajectory parameters, and text prompts.
Training uses the AdamW optimizer in bfloat16 mixed precision on $64 \times \text{NVIDIA H20}$ GPUs for $20\text{k}$ iterations with a global batch size of $64$ (details in Appendix~\ref{app:details-training}).

\paragraph{Evaluation protocols.}
We evaluate driving scenes on $128$ representative sequences selected from the nuScenes validation and test sets to cover diverse trajectory geometries. Each base trajectory is augmented with left-turn and right-turn variations, yielding $384$ test trajectories in total.
Following standard visual odometry evaluation, camera poses are extracted from generated frames using VGGT-$\Omega$~\citep{wang2026vggtomega}.
Quantitative metrics report geodesic rotation error ($\text{rot}^\circ$), scale-aligned relative translation error ($\text{tr}\%$), pairwise joint-pose AUC ($\text{AUC}@3$, $\text{AUC}@10$), and camera motion consistency ($\text{CamMC}$).
Visual quality is evaluated separately on $512$ nuScenes clips using FID over $20{,}480$ predicted frames and FVD over $512$ generated videos after excluding the $9$-frame conditioning prefix.
Appendix~\ref{app:details-metrics} defines all metrics mathematically.

\paragraph{Matched-backbone baselines.}
For a matched-backbone comparison, all methods share the same Wan2.2 TI2V-5B backbone~\citep{wan2025,wan2025wan22}, initialization, training mixture, optimization schedule, and attention head dimension ($d = 128$).
All camera-conditioned methods are retrained with the same $15$-block alternating injection schedule.
GTA, PRoPE, and UCPE use a common fixed global translation scale estimated solely from training-set motion statistics.
We compare against: (i)~the backbone without camera PE, (ii)~GTA~\citep{miyato2024gta}, (iii)~PRoPE~\citep{li2025prope}, (iv)~UCPE~\citep{zhang2026ucpe}, (v)~RayNova~\citep{xie2026raynova}, (vi)~RayRoPE~\citep{wu2026rayrope}, (vii)~URoPE~\citep{xie2026urope}, and (viii)~a CameraCtrl-style baseline~\citep{he2025cameractrl}, adapted to the DiT backbone by adding ray features before the query and key projections following Lyra~2.0~\citep{shen2026lyra2}.

\subsection{Pose-Controlled Video Generation on Driving Benchmarks}
\label{sec:exp_video}

\begin{table*}[t]
\centering
\caption{\textbf{Camera pose controllability on the nuScenes driving benchmark.} Camera-conditioned methods are ordered by CamMC from highest to lowest. Best in \textbf{bold}; second-best is underlined.}
\label{tab:main-video}
\setlength{\tabcolsep}{3.5pt}
\small
\begin{tabular}{lccccccc}
\toprule
Method & rot$^\circ$$\downarrow$ & tr\%$\,\downarrow$
& AUC@3$\uparrow$ & AUC@10$\uparrow$ & CamMC$\downarrow$ & FID$\downarrow$ & FVD$\downarrow$ \\
\midrule
No camera PE & $5.57$ & $5.94$ & $17.21$ & $52.05$ & $7.96$ & $21.06$ & $153.31$ \\
\midrule
RayRoPE~\citep{wu2026rayrope} & $1.83$ & $2.97$ & $36.56$ & $76.62$ & $2.78$ & $21.21$ & $150.26$ \\
PRoPE~\citep{li2025prope} & $1.53$ & $3.10$ & $42.11$ & $78.58$ & $2.67$ & $20.40$ & $142.16$ \\
RayNova~\citep{xie2026raynova} & $1.56$ & $3.02$ & $41.19$ & $78.42$ & $2.66$ & $20.15$ & $144.91$ \\
CameraCtrl-style~\citep{he2025cameractrl} & $1.54$ & $3.07$ & \underline{$42.57$} & $78.90$ & $2.63$ & $20.13$ & $142.23$ \\
URoPE~\citep{xie2026urope} & $1.66$ & $\mathbf{2.67}$ & $\mathbf{44.41}$ & $\mathbf{79.92}$ & $2.53$ & \underline{$19.98$} & \underline{$135.63$} \\
GTA~\citep{miyato2024gta} & $1.44$ & $3.03$ & $40.53$ & $77.72$ & $2.49$ & $20.00$ & $141.84$ \\
UCPE~\citep{zhang2026ucpe} & \underline{$1.39$} & $2.95$ & $41.21$ & $78.38$ & \underline{$2.45$} & $20.20$ & $136.40$ \\
\method & $\mathbf{1.37}$ & \underline{$2.84$}
& $42.11$ & \underline{$79.47$} & $\mathbf{2.32}$ & $\mathbf{19.82}$ & $\mathbf{134.15}$ \\
\bottomrule
\end{tabular}
\end{table*}

\paragraph{Main comparison on driving benchmarks.}
Table~\ref{tab:main-video} compares camera control across methods.  We use CamMC as the primary metric because it jointly evaluates rotation and translation trajectories, whereas thresholded AUC is less stable and less sensitive to translation magnitude.
\method records the lowest CamMC ($2.32$); UCPE is the strongest baseline at $2.45$.
It also achieves the best rotation accuracy and ranks second on both tr\% and AUC@10.
URoPE leads tr\% and the two AUC metrics, while the CameraCtrl-style baseline ranks second on AUC@3.
The weaker pose results without camera PE show the benefit of explicit camera geometry for controllable generation.
\method also attains the lowest FID and FVD point estimates, so its control gains do not coincide with worse measured visual quality.
Figure~\ref{fig:qual-pose} compares the commanded and recovered paths under three camera controls together with the corresponding generated frames.
Appendices~\ref{app:qual-sota} and~\ref{app:qual-pose-more} provide matched-backbone comparisons and additional nuScenes scenes, respectively.

\begin{figure*}[t]
  \centering
  \includegraphics[width=\textwidth]{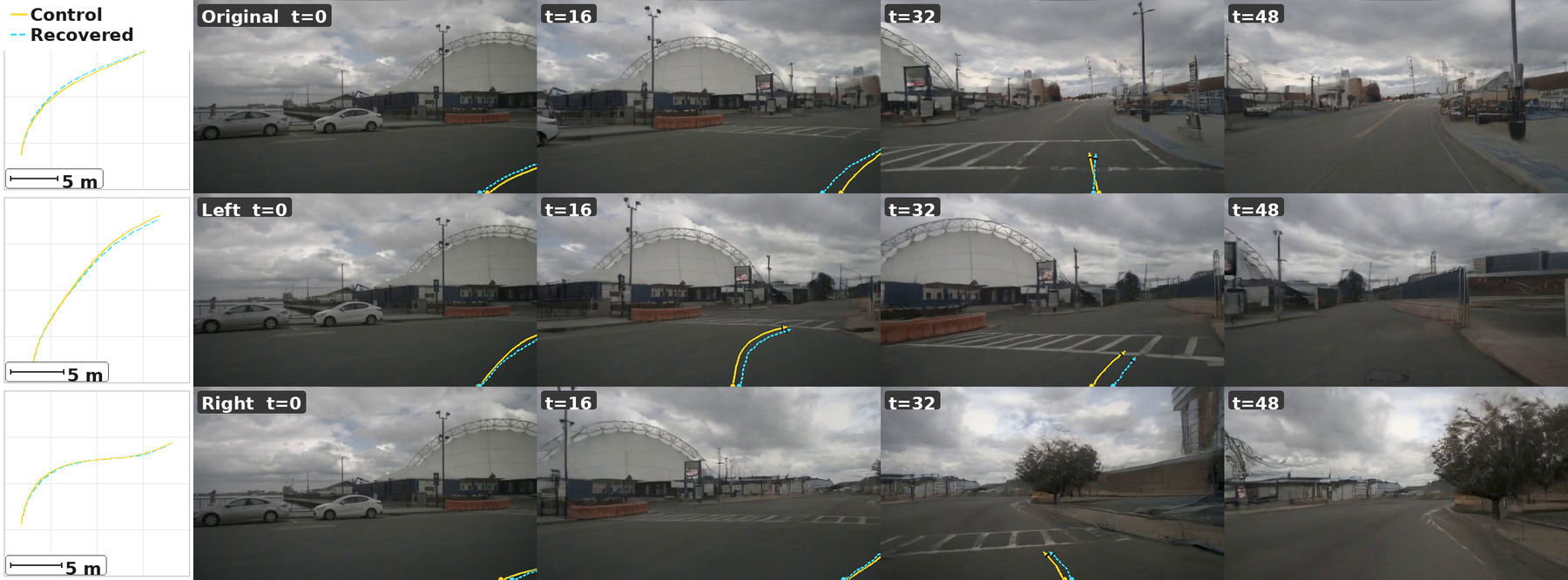}\\
  \includegraphics[width=\textwidth]{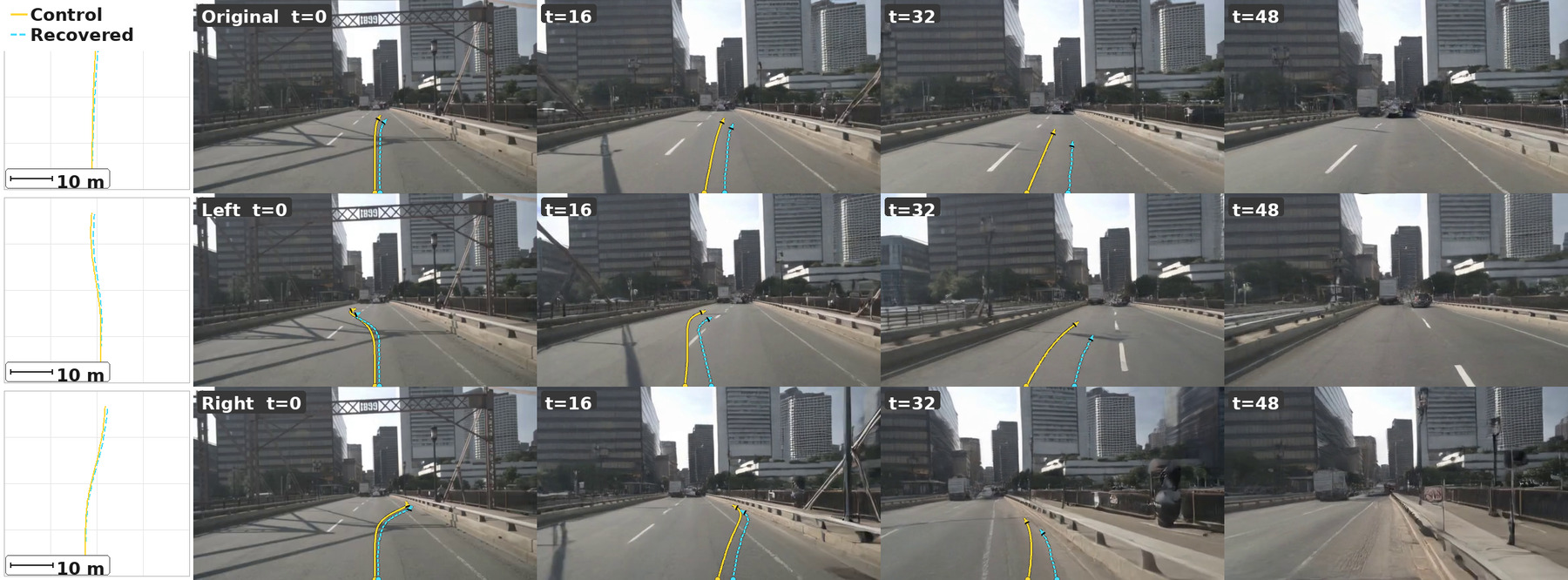}
  \caption{\textbf{Qualitative video generation under camera pose control.} Visual results on two representative nuScenes scenes under three commanded trajectories (original, left turn, right turn). In each row, the leftmost plot displays the bird's-eye view of the commanded path (solid) and the path recovered from the generated video by VGGT-$\Omega$ (dashed), followed by generated video frames sampled at $t \in \{0, 16, 32, 48\}$.}
  \label{fig:qual-pose}
\end{figure*}

\paragraph{Ablation of encoding components.}
Table~\ref{tab:encoding-ablation} shows that \method's three sub-blocks play distinct, complementary roles while keeping $d=128$ fixed through identity padding.
Removing metric translation degrades translation fidelity: all variants without it have substantially worse tr\% and CamMC, even though Disparity + rot. attains the highest AUC because AUC evaluates translation direction but not magnitude.
Ray-local rotation primarily controls orientation: adding it to Disparity + trans. reduces rotation error from $1.66^\circ$ to $1.37^\circ$.
Disparity anchors instead refine cross-view tracking: adding them to Rot. + trans. improves both AUC metrics and lowers tr\% and CamMC while leaving rotation essentially unchanged.
Together, these complementary gains yield the full model's lowest tr\% and CamMC.

\begin{table*}[t]
\centering
\caption{\textbf{Ablation of camera-encoding components on nuScenes.} Best in \textbf{bold}; second-best is underlined.}
\label{tab:encoding-ablation}
\scriptsize
\setlength{\tabcolsep}{2.5pt}
\begin{tabular}{lccc ccccc}
\toprule
& \multicolumn{3}{c}{Components} & \multicolumn{5}{c}{Overall Metrics} \\
\cmidrule(lr){2-4} \cmidrule(lr){5-9}
Method & camera rot. & metric trans. & disparity anchor
& rot$^\circ$$\downarrow$ & tr\%$\,\downarrow$
& AUC@3$\uparrow$ & AUC@10$\uparrow$ & CamMC$\downarrow$ \\
\midrule
\method & $\checkmark$ & $\checkmark$ & $\checkmark$
& \underline{$1.37$} & $\mathbf{2.84}$ & $42.11$ & $79.47$ & $\mathbf{2.32}$ \\
Rot. + trans. & $\checkmark$ & $\checkmark$ &
& $\mathbf{1.36}$ & $3.09$ & $38.82$ & $77.50$ & \underline{$2.40$} \\
Disparity + trans. & & $\checkmark$ & $\checkmark$
& $1.66$ & \underline{$2.88$} & $40.60$ & $78.82$ & $2.56$ \\
Disparity + rot. & $\checkmark$ & & $\checkmark$
& $1.48$ & $4.31$ & $\mathbf{44.42}$ & $\mathbf{79.78}$ & $3.66$ \\
Disparity only & & & $\checkmark$
& $1.59$ & $4.01$ & \underline{$43.19$} & \underline{$79.53$} & $3.56$ \\
Rotation only & $\checkmark$ & &
& $1.41$ & $4.01$ & $42.52$ & $78.33$ & $3.35$ \\
Translation only & & $\checkmark$ &
& $2.59$ & $3.09$ & $30.64$ & $73.50$ & $3.50$ \\
\bottomrule
\end{tabular}
\end{table*}

Appendices~\ref{app:architecture} and~\ref{app:history} analyze adapter capacity, injection strategies, layer-wise modulation, and real-world grounding with retrieved history.

\subsection{Camera-Diverse Evaluation on PanShot}
\label{sec:exp_panshot}

To complement the fixed-intrinsic nuScenes setting, we evaluate \method on the PanShot benchmark~\citep{zhang2026ucpe,zhang2026ucpecode}, which spans pinhole, wide-angle, and fisheye lenses under the unified camera model.
Following the matched protocol, we retrain every compared method from the same Wan2.1-T2V-1.3B backbone for $10\text{k}$ steps on $8 \times \text{NVIDIA H20}$ GPUs with identical spatial attention adapters and evaluate them on the official jitter-filtered test split ($272$ clips).
We compare camera-conditioning poses with near-depth normalization (a)~enabled or (b)~disabled; the latter preserves raw physical displacements during training and inference, while evaluation still applies the per-trajectory translation normalization in Appendix~\ref{app:details-metrics}.
Poses are extracted using VGGT-$\Omega$~\citep{wang2026vggtomega}.

\begin{table*}[t]
\centering
\caption{\textbf{Camera pose control on PanShot across diverse camera optics.} Evaluated on $272$ test clips under the unified camera model ($C=8$, $1\times 192$, VGGT-$\Omega$). (a)~Conditioning-pose near-depth normalization enabled. (b)~Conditioning-pose near-depth normalization disabled. GTA and PRoPE retain the latitude-up map; UCPE and \method do not. Lower is better. Best in each panel in \textbf{bold}; second-best is underlined.}
\label{tab:panshot-pose}
\small
\setlength{\tabcolsep}{4pt}
\begin{minipage}[t]{0.48\linewidth}
\centering
\textbf{(a)} Conditioning norm. on
\begin{tabular}{lccc}
\toprule
Method & RotErr$\downarrow$ & TransErr$\downarrow$ & CamMC$\downarrow$ \\
\midrule
GTA   & $5.14$ & $19.85$ & $22.59$ \\
PRoPE & $5.10$ & $20.38$ & $23.02$ \\
UCPE  & $\mathbf{4.22}$ & $\underline{15.15}$ & $\underline{17.50}$ \\
\method & $\underline{4.57}$ & $\mathbf{10.99}$ & $\mathbf{13.88}$ \\
\bottomrule
\end{tabular}
\end{minipage}
\hfill
\begin{minipage}[t]{0.48\linewidth}
\centering
\textbf{(b)} Conditioning norm. off
\begin{tabular}{lccc}
\toprule
Method & RotErr$\downarrow$ & TransErr$\downarrow$ & CamMC$\downarrow$ \\
\midrule
GTA   & $5.30$ & $22.70$ & $25.34$ \\
PRoPE & $5.11$ & $23.04$ & $25.58$ \\
UCPE  & $\mathbf{4.40}$ & $\underline{17.91}$ & $\underline{20.13}$ \\
\method & $\underline{4.82}$ & $\mathbf{12.66}$ & $\mathbf{15.56}$ \\
\bottomrule
\end{tabular}
\end{minipage}
\end{table*}

Table~\ref{tab:panshot-pose} shows that \method's advantage on PanShot lies in translation control rather than rotation.
Across both conditioning-pose settings, it achieves the lowest TransErr and CamMC, while UCPE and \method remain close on RotErr.
The wider separation between UCPE and PRoPE than on fixed-intrinsic nuScenes is consistent with the benefit of calibrated ray representations when camera optics vary.
Using raw conditioning displacements during training and generation increases TransErr and CamMC for all methods, but \method retains the lowest errors and the smallest increase.
Figure~\ref{fig:qual-panshot} shows generated sequences across varying fields of view and lens distortions.

\begin{figure*}[!ht]
  \centering
  \includegraphics[width=\textwidth]{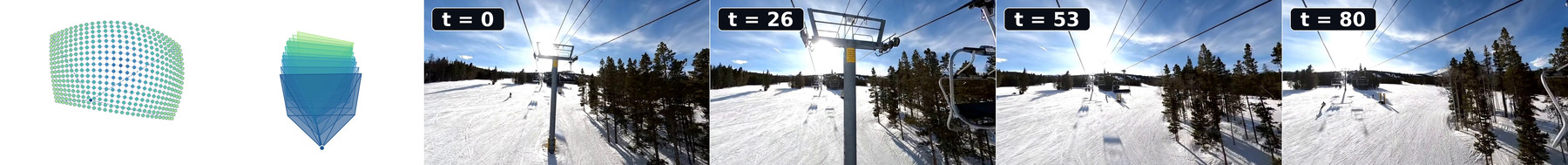}\\[2pt]
  \includegraphics[width=\textwidth]{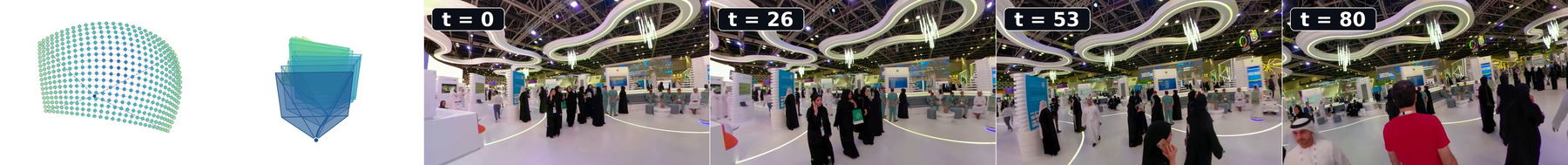}\\[2pt]
  \includegraphics[width=\textwidth]{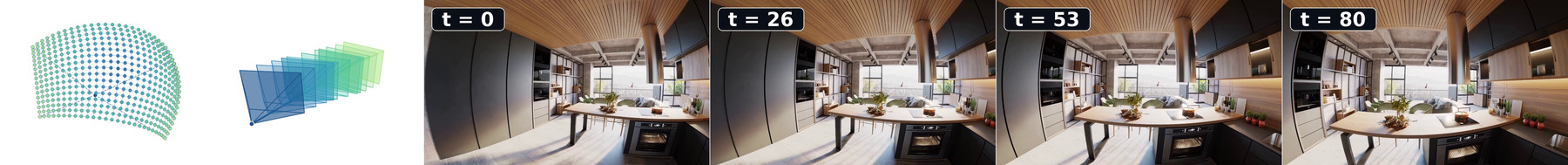}\\[2pt]
  \includegraphics[width=\textwidth]{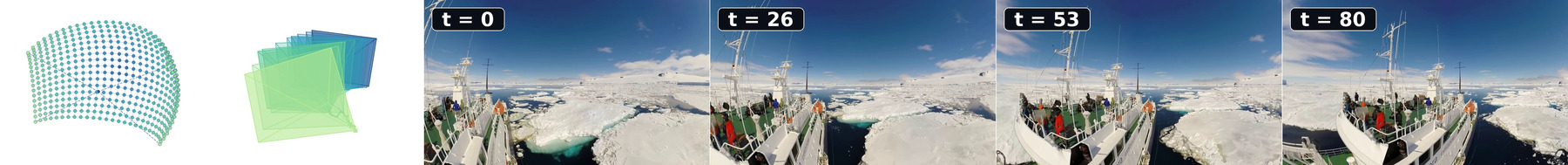}
  \caption{\textbf{PanShot generations across camera optics.} Rows~1--4 span $(\text{FoV},\xi)$ from $(97^\circ,0.00)$ to $(173^\circ,1.66)$, covering pinhole through fisheye cameras.}
  \label{fig:qual-panshot}
\end{figure*}
\FloatBarrier

\section{Conclusion}
\label{sec:conclusion}

We introduced \method, a norm-preserving relative camera encoding for camera-controlled video generation under real-world metric motion.  Our analysis identifies baseline-dependent logit and feature-norm growth in homogeneous projective encodings and formalizes a three-way trade-off: retaining the complete metric relative pose, factorizing strictly per token, and preserving feature norms cannot be achieved simultaneously.  \method resolves this trade-off by relaxing per-token factorization through query-camera grouping, representing relative ray orientation and raw metric translation with orthogonal blocks, and adding disparity-anchored spherical rotations as geometric knots for cross-view relations.

Across large-baseline nuScenes and camera-diverse PanShot, \method achieves the lowest CamMC among the compared encodings; it also attains the lowest nuScenes FID/FVD point estimates and the lowest PanShot TransErr under both conditioning settings.  These results support norm-preserving metric camera encoding as a stable basis for camera control across physical baselines and camera optics.  Query-camera grouping introduces additional local camera-attention cost; reducing this overhead while retaining metric sensitivity is an important direction for future work.

\subsection*{AI use statement}
Generative AI tools assisted code implementation and language editing. All outputs were reviewed and verified by the authors, who take full responsibility for this paper.

\clearpage
\bibliography{references}
\bibliographystyle{iclr2027_conference}

\appendix
\section{Impossibility of Unitary Relative Per-Token Encodings for Metric Translation}
\label{app:per-token}

We show that no strictly per-token, unitary (or orthogonal) positional operator can completely capture metric relative translation. This result formalizes the trilemma introduced in Section~\ref{sec:intro} and motivates \method's use of query-camera grouping.

\subsection{Problem Formulation and Hypotheses}

Let $T_i = (R_i, \mathbf{o}_i) \in SE(3)$ denote the camera-to-world pose of view $i$, where $R_i \in \mathrm{SO}(3)$ and $\mathbf{o}_i \in \mathbb{R}^3$. The intrinsic relative pose between views $i$ and $j$, invariant to any global coordinate change $G \in SE(3)$, is given by
\begin{equation}
T_i^{-1}T_j = \big(R_i^\top R_j,\, R_i^\top(\mathbf{o}_j - \mathbf{o}_i)\big).
\label{eq:relative-pose-se3}
\end{equation}
We examine the existence of a matrix-valued positional embedding operator $U: SE(3) \to \mathbb{C}^{d\times d}$ (or $\mathbb{R}^{d\times d}$) satisfying four standard axioms of relative positional attention:
\begin{enumerate}
    \item[\textnormal{(H1)}] \textbf{Multiplicative per-token factorization:} The camera positional operator acts independently on each query and key token via its own pose:
    \begin{equation}
    \widetilde{\mathbf{q}}_i = U(T_i)\mathbf{q}_i, \qquad \widetilde{\mathbf{k}}_j = U(T_j)\mathbf{k}_j, \qquad s_{ij} = \frac{\mathbf{q}_i^\ast U(T_i)^\ast U(T_j) \mathbf{k}_j}{\sqrt{d}}.
    \end{equation}
    \item[\textnormal{(H2)}] \textbf{Exact relative-pose dependence:} The pairwise interaction depends strictly and exclusively on the relative pose $T_i^{-1}T_j$:
    \begin{equation}
    U(T_i)^\ast U(T_j) = U(T_i^{-1}T_j), \qquad \forall T_i, T_j \in SE(3).
    \label{eq:relative-axiom}
    \end{equation}
    \item[\textnormal{(H3)}] \textbf{Norm preservation (unitarity):} $U(T)$ is unitary (or orthogonal in the real case), so that $\|U(T)\mathbf{x}\|_2 = \|\mathbf{x}\|_2$ for all $\mathbf{x} \in \mathbb{C}^d$, i.e.,
    \begin{equation}
    U(T)^\ast U(T) = I_d, \qquad \forall T \in SE(3).
    \label{eq:unitarity-axiom}
    \end{equation}
    \item[\textnormal{(H4)}] \textbf{Finite-dimensional continuity:} The feature dimension is finite ($d < \infty$), and the mapping $T \mapsto U(T)$ is continuous with respect to the standard Lie group topology on $SE(3)$.
\end{enumerate}

\subsection{The Impossibility Theorem}

\begin{theorem}[Translation Blindness of Unitary Relative Per-Token Encodings]
\label{thm:translation-blindness}
Let $U: SE(3) \to \mathrm{U}(d)$ be any continuous map satisfying hypotheses \textnormal{(H1)}--\textnormal{(H4)}. Then for every pure metric translation $\mathbf{t} \in \mathbb{R}^3$,
\begin{equation}
U(I_3, \mathbf{t}) = I_d.
\end{equation}
Consequently, for any arbitrary camera pose $(R, \mathbf{t}) \in SE(3)$,
\begin{equation}
U(R, \mathbf{t}) = U(R, \mathbf{0}).
\end{equation}
The positional operator $U$ is identically trivial on metric translations, failing to distinguish any non-zero relative translation $\Delta\mathbf{o}_{j\mid i} = R_i^\top(\mathbf{o}_j - \mathbf{o}_i) \neq \mathbf{0}$.
\end{theorem}

\begin{proof}
\noindent\textbf{Step 1: Homomorphism deduction.}
First, we show that (H2) and (H3) enforce $U$ to be a group representation of $SE(3)$, i.e., $U(gh) = U(g)U(h)$ for all $g, h \in SE(3)$:
\begin{itemize}
    \item Setting $T_i = T_j = I$ in Eq.~\eqref{eq:relative-axiom} yields
    $U(I)^\ast U(I) = U(I)$. On the other hand,
    Eq.~\eqref{eq:unitarity-axiom} with $T = I$ gives
    $U(I)^\ast U(I) = I_d$. The two right-hand sides
    must coincide, so $U(I) = I_d$.
    \item Setting $T_j = I$ in Eq.~\eqref{eq:relative-axiom} yields
    $U(T_i)^\ast U(I) = U(T_i^{-1})$. Since $U(I) = I_d$ from the first
    item, the left-hand side reduces to $U(T_i)^\ast$, and by unitarity
    ($U(T_i)^{-1} = U(T_i)^\ast$) we conclude
    $U(T_i^{-1}) = U(T_i)^\ast = U(T_i)^{-1}$.
    \item For any $g, h \in SE(3)$, set $T_i = g$ and $T_j = gh$. Substituting into Eq.~\eqref{eq:relative-axiom} gives:
    \begin{equation}
    U(g)^\ast U(gh) = U(g^{-1}(gh)) = U(h).
    \end{equation}
    Multiplying on the left by $U(g)$ and using $U(g)U(g)^\ast = I_d$, we conclude that
    \begin{equation}
    U(gh) = U(g)U(h), \qquad \forall g, h \in SE(3).
    \label{eq:homomorphism}
    \end{equation}
\end{itemize}

\noindent\textbf{Step 2: Commutativity of the translation subgroup.}
Let $U_{\mathrm{tr}}(\mathbf{t}) := U(I_3, \mathbf{t})$ denote the pure translation operator for $\mathbf{t} \in \mathbb{R}^3$. In $SE(3)$, pure translations commute: $(I_3, \mathbf{s})(I_3, \mathbf{t}) = (I_3, \mathbf{s}+\mathbf{t}) = (I_3, \mathbf{t})(I_3, \mathbf{s})$. Applying the homomorphism in Eq.~\eqref{eq:homomorphism}:
\begin{equation}
U_{\mathrm{tr}}(\mathbf{s})U_{\mathrm{tr}}(\mathbf{t})
= U\big((I_3,\mathbf{s})(I_3,\mathbf{t})\big)
= U(I_3,\mathbf{s}+\mathbf{t})
= U\big((I_3,\mathbf{t})(I_3,\mathbf{s})\big)
= U_{\mathrm{tr}}(\mathbf{t})U_{\mathrm{tr}}(\mathbf{s})
= U_{\mathrm{tr}}(\mathbf{s}+\mathbf{t}).
\label{eq:translation-commuting}
\end{equation}
Thus, $\{U_{\mathrm{tr}}(\mathbf{t}) : \mathbf{t} \in \mathbb{R}^3\}$ is a family of mutually commuting unitary matrices.

\noindent\textbf{Step 3: Simultaneous unitary diagonalization.}
Every unitary matrix is normal ($U^\ast U = UU^\ast = I$). Step~2 shows that the operators $\{U_{\mathrm{tr}}(\mathbf{t}) : \mathbf{t} \in \mathbb{R}^3\}$ pairwise commute, and each is unitary by~\textnormal{(H3)}, hence normal. The spectral theorem for commuting families of normal operators~\citep[Theorem~2.5.5]{horn2012matrix}, which applies to families of arbitrary cardinality, then guarantees a single, $\mathbf{t}$-independent unitary change-of-basis matrix $V \in \mathrm{U}(d)$ that simultaneously diagonalizes all $U_{\mathrm{tr}}(\mathbf{t})$:
\begin{equation}
V^\ast U_{\mathrm{tr}}(\mathbf{t}) V = \operatorname{diag}\big(\chi_1(\mathbf{t}), \chi_2(\mathbf{t}), \dots, \chi_d(\mathbf{t})\big).
\label{eq:simultaneous-diag}
\end{equation}
By (H4), for each $a \in \{1,\dots,d\}$, the diagonal entry $\chi_a: \mathbb{R}^3 \to \mathbb{C}$ is continuous. By (H3), $|\chi_a(\mathbf{t})| = 1$. It remains to show that each entry is multiplicative in $\mathbf{t}$. Conjugating Eq.~\eqref{eq:translation-commuting} by $V$ and inserting $VV^\ast = I_d$ between $U_{\mathrm{tr}}(\mathbf{s})$ and $U_{\mathrm{tr}}(\mathbf{t})$:
\begin{equation}
V^\ast U_{\mathrm{tr}}(\mathbf{s}+\mathbf{t}) V
= \big(V^\ast U_{\mathrm{tr}}(\mathbf{s}) V\big)\big(V^\ast U_{\mathrm{tr}}(\mathbf{t}) V\big).
\label{eq:conj-commuting}
\end{equation}
By Eq.~\eqref{eq:simultaneous-diag}, the left-hand side equals $\operatorname{diag}\big(\chi_1(\mathbf{s}+\mathbf{t}), \dots, \chi_d(\mathbf{s}+\mathbf{t})\big)$, and the right-hand side equals $\operatorname{diag}\big(\chi_1(\mathbf{s})\chi_1(\mathbf{t}), \dots, \chi_d(\mathbf{s})\chi_d(\mathbf{t})\big)$. Comparing diagonal entries, each entry satisfies the multiplicative Cauchy character equation:
\begin{equation}
\chi_a(\mathbf{s}+\mathbf{t}) = \chi_a(\mathbf{s})\chi_a(\mathbf{t}), \qquad \forall \mathbf{s}, \mathbf{t} \in \mathbb{R}^3.
\end{equation}

\noindent\textbf{Step 4: Character form and finite spatial spectrum.}
Fixing an index $a$ and dropping it, we classify the continuous maps $\chi:(\mathbb{R}^3,+)\to\mathbb{T}$, where $\mathbb{T}=\{z\in\mathbb{C}:|z|=1\}$, satisfying the multiplicative character equation
\begin{equation}
\chi(\mathbf{s}+\mathbf{t})=\chi(\mathbf{s})\chi(\mathbf{t}), \qquad |\chi|=1 ,
\label{eq:character-equation}
\end{equation}
whose only solutions are spatial plane waves; this is the classical
classification of continuous characters of $(\mathbb{R}^3,+)$~\citep[Theorem~4.6 and Corollary~4.8]{folland2016course}:
\begin{equation}
\chi_a(\mathbf{t}) = \exp(i\boldsymbol{\omega}_a^\top \mathbf{t}), \qquad \text{for a unique frequency vector } \boldsymbol{\omega}_a \in \mathbb{R}^3.
\end{equation}
This classification enables the spherical-orbit contradiction below.
Consequently, in the eigenbasis of $V$, every translation operator takes the standard multi-frequency phase form:
\begin{equation}
V^\ast U_{\mathrm{tr}}(\mathbf{t}) V = \operatorname{diag}\big(\exp(i\boldsymbol{\omega}_1^\top \mathbf{t}), \dots, \exp(i\boldsymbol{\omega}_d^\top \mathbf{t})\big).
\end{equation}
Because the feature dimension is finite ($d < \infty$), the active frequency spectrum
\begin{equation}
\Omega := \{\boldsymbol{\omega}_1, \boldsymbol{\omega}_2, \dots, \boldsymbol{\omega}_d\} \subset \mathbb{R}^3
\end{equation}
contains at most $d$ discrete frequency vectors ($|\Omega| \le d$).

\noindent\textbf{Step 5: Semi-direct product conjugation and spherical orbit contradiction.}
\paragraph{A lemma on rotation invariance.} If a finite set $\Omega \subset \mathbb{R}^3$ is $\mathrm{SO}(3)$-invariant, i.e.\ $R\Omega = \Omega$ for all $R \in \mathrm{SO}(3)$, then $\Omega = \{\mathbf{0}\}$. Indeed, the orbit of any non-zero $\boldsymbol{\omega} \in \Omega$ under $\mathrm{SO}(3)$ is the entire sphere $\{\mathbf{u} \in \mathbb{R}^3 : \|\mathbf{u}\|_2 = \|\boldsymbol{\omega}\|_2\}$, an uncountably infinite set; invariance forces this whole sphere into $\Omega$, which finiteness forbids.

Let $U_{\mathrm{rot}}(R) := U(R, \mathbf{0})$ denote the pure rotation operator for $R \in \mathrm{SO}(3)$. In $SE(3)$, rotations and translations obey the semi-direct product conjugation:
\begin{equation}
(I_3, R\mathbf{t}) = (R, \mathbf{0})(I_3, \mathbf{t})(R, \mathbf{0})^{-1}, \qquad \forall R \in \mathrm{SO}(3), \ \mathbf{t} \in \mathbb{R}^3.
\label{eq:se3-conjugation}
\end{equation}
Applying the homomorphism $U$ to Eq.~\eqref{eq:se3-conjugation} gives the
operator-level conjugation
\begin{equation}
U_{\mathrm{tr}}(R\mathbf{t}) = U_{\mathrm{rot}}(R) U_{\mathrm{tr}}(\mathbf{t}) U_{\mathrm{rot}}(R)^{-1}.
\label{eq:op-conjugation}
\end{equation}
Since $\mathbf{t}$ is a free variable ranging over $\mathbb{R}^3$, replacing it by
$R^{-1}\mathbf{t}$ and then multiplying both sides on the right by $U_{\mathrm{rot}}(R)$
(the factors $U_{\mathrm{rot}}(R)^{-1}U_{\mathrm{rot}}(R)=I_d$ cancel) recasts the same
fact in the commutation form
\begin{equation}
U_{\mathrm{tr}}(\mathbf{t}) U_{\mathrm{rot}}(R) = U_{\mathrm{rot}}(R) U_{\mathrm{tr}}(R^{-1}\mathbf{t}),
\label{eq:op-commutation}
\end{equation}
which is the form we use below to push $U_{\mathrm{rot}}(R)$ past $U_{\mathrm{tr}}(\mathbf{t})$ when acting on an eigenvector.
For any frequency $\boldsymbol{\omega} \in \Omega$, define its joint eigenspace, and set $E_{\boldsymbol{\omega}}=\{\mathbf{0}\}$ for $\boldsymbol{\omega}\notin\Omega$:
\begin{equation}
E_{\boldsymbol{\omega}} = \big\{ \mathbf{v} \in \mathbb{C}^d : U_{\mathrm{tr}}(\mathbf{t})\mathbf{v} = \exp(i\boldsymbol{\omega}^\top \mathbf{t})\mathbf{v}, \ \forall \mathbf{t} \in \mathbb{R}^3 \big\},
\quad
E_{\boldsymbol{\omega}}=\{\mathbf{0}\}\ \text{if}\ \boldsymbol{\omega}\notin\Omega .
\end{equation}
For any non-zero $\mathbf{v} \in E_{\boldsymbol{\omega}}$, we apply Eq.~\eqref{eq:op-commutation} to evaluate the translation action on $U_{\mathrm{rot}}(R)\mathbf{v}$:
\begin{equation}
\begin{aligned}
U_{\mathrm{tr}}(\mathbf{t}) \big(U_{\mathrm{rot}}(R)\mathbf{v}\big)
&= U_{\mathrm{rot}}(R) U_{\mathrm{tr}}(R^{-1}\mathbf{t})\mathbf{v} \\
&= U_{\mathrm{rot}}(R) \exp(i\boldsymbol{\omega}^\top R^{-1}\mathbf{t})\mathbf{v} \\
&= \exp(i(R\boldsymbol{\omega})^\top \mathbf{t}) \big(U_{\mathrm{rot}}(R)\mathbf{v}\big).
\end{aligned}
\end{equation}
Here the second equality uses that $\mathbf{v}\in E_{\boldsymbol{\omega}}$ is an
eigenvector of $U_{\mathrm{tr}}(R^{-1}\mathbf{t})$; the last uses that the
(unit-modulus) eigenvalue $\exp(i\boldsymbol{\omega}^\top R^{-1}\mathbf{t})$ is a
complex scalar, hence commutes with $U_{\mathrm{rot}}(R)$ and can be pulled out
of the product, together with $\boldsymbol{\omega}^\top R^{-1}\mathbf{t}=(R\boldsymbol{\omega})^\top\mathbf{t}$.
Thus, $U_{\mathrm{rot}}(R)\mathbf{v}$ is an eigenvector of $U_{\mathrm{tr}}(\mathbf{t})$ with spatial frequency $R\boldsymbol{\omega}$, meaning $U_{\mathrm{rot}}(R)E_{\boldsymbol{\omega}} = E_{R\boldsymbol{\omega}}$.
This necessitates that the frequency spectrum $\Omega$ must be invariant under all continuous 3D rotations:
\begin{equation}
R\Omega = \Omega, \qquad \forall R \in \mathrm{SO}(3).
\label{eq:so3-invariance}
\end{equation}
By Eq.~\eqref{eq:so3-invariance} and the finiteness $|\Omega| \le d < \infty$ from Step~4, $\Omega$ is a finite $\mathrm{SO}(3)$-invariant subset of $\mathbb{R}^3$; the lemma therefore forces $\Omega = \{\mathbf{0}\}$.

Substituting $\boldsymbol{\omega}_a = \mathbf{0}$ back into the diagonalized form yields
\begin{equation}
U_{\mathrm{tr}}(\mathbf{t}) = V I_d V^\ast = I_d, \qquad \forall \mathbf{t} \in \mathbb{R}^3.
\end{equation}
Finally, decomposing any arbitrary pose as $(R, \mathbf{t}) = (I_3, \mathbf{t})(R, \mathbf{0})$ and applying the homomorphism yields $U(R, \mathbf{t}) = U_{\mathrm{tr}}(\mathbf{t})U_{\mathrm{rot}}(R) = U(R, \mathbf{0})$. This completes the proof.
\end{proof}

\section{Implementation and Evaluation Details}
\label{app:details}

\subsection{Architecture and Training Configurations}
\label{app:details-training}

\paragraph{Head partition.}
In our primary implementation based on the Wan2.2 TI2V-5B backbone ($d = 128$), the channel budget is partitioned into four orthogonal sub-blocks: $d = d_{\mathrm{disp}} + d_{\mathrm{rot}} + d_{\mathrm{trans}} + d_{\mathrm{native}} = 36 + 36 + 24 + 32$. Specifically, $d_{\mathrm{disp}}=36$ encodes $L=6$ non-uniform angular disparity fractions $\rho_\ell \in \{0, 0.05, 0.15, 0.4, 0.7, 1\}$ with $n=2$ triplets each ($6 \times 2 \times 3$); $d_{\mathrm{rot}}=36$ allocates $m=12$ coordinate triplets for ray-local MinRot; $d_{\mathrm{trans}}=24$ covers $3$ spatial axes over $K=4$ log-spaced metric wavelengths ($\lambda \in [0.5, 200]$\,m); and $d_{\mathrm{native}}=32$ applies standard RoPE to temporal ($16$D) and image-plane $x/y$ ($8$D each) coordinate bands.

\paragraph{Optimization details.}
All driving models are trained with AdamW ($\beta_1 = 0.9$, $\beta_2 = 0.999$, $\epsilon = 10^{-8}$, weight decay $0.01$) using a constant learning rate of $1 \times 10^{-5}$ and no warmup. Training uses bfloat16 mixed precision on $64 \times \text{NVIDIA H20}$ GPUs for $20\text{k}$ iterations with a global batch size of $64$. PanShot models are trained on $8 \times \text{NVIDIA H20}$ GPUs for $10\text{k}$ iterations.

\subsection{Evaluation Metrics and Mathematical Formulations}
\label{app:details-metrics}

Estimated camera poses are recovered from generated frames by VGGT-$\Omega$~\citep{wang2026vggtomega}.
Let $T$ denote the number of frames in an evaluated sequence, and let $T_t = (R_t, \mathbf{o}_t) \in SE(3)$ and $\widehat{T}_t = (\widehat{R}_t, \widehat{\mathbf{o}}_t) \in SE(3)$ denote the commanded and estimated camera-to-world poses at frame $t \in \{1, \dots, T\}$.
To remove the absolute $SE(3)$ gauge, all poses are expressed relative to the first frame,
\begin{equation}
T_t^{\mathrm{rel}} = T_1^{-1} T_t = \big(R_1^\top R_t,\, R_1^\top(\mathbf{o}_t - \mathbf{o}_1)\big) =: \big(R_t^{\mathrm{rel}},\, \mathbf{o}_t^{\mathrm{rel}}\big),
\label{eq:eval-rel-pose}
\end{equation}
with $\widehat{T}_t^{\mathrm{rel}} = (\widehat{R}_t^{\mathrm{rel}}, \widehat{\mathbf{o}}_t^{\mathrm{rel}})$ constructed analogously.
The geodesic rotation error in degrees is
\begin{equation}
\angle(R,\widehat{R})
= \arccos\!\left(\operatorname{clamp}\!\left(\frac{\operatorname{tr}(R^\top\widehat{R})-1}{2},-1,1\right)\right)\times\frac{180^\circ}{\pi}.
\label{eq:eval-geodesic}
\end{equation}
Reported driving scores average equally over the $384$ test clips; PanShot scores average over the $272$ jitter-filtered clips.
The pairwise AUC below is a micro-average over all within-clip frame pairs.

Monocular pose recovery leaves a global scale ambiguity.
On driving evaluations we therefore apply a single least-squares scale to the estimated camera centers, leaving rotations unchanged:
\begin{equation}
s^\star
= \frac{\sum_{t=1}^{T} \widehat{\mathbf{o}}_t^{\mathrm{rel}\top}\,\mathbf{o}_t^{\mathrm{rel}}}{\sum_{t=1}^{T} \|\widehat{\mathbf{o}}_t^{\mathrm{rel}}\|_2^2},
\qquad
\widetilde{\mathbf{o}}_t
= s^\star\,\widehat{\mathbf{o}}_t^{\mathrm{rel}}.
\label{eq:eval-scale}
\end{equation}
Following the official PanShot evaluation code~\citep{zhang2026ucpecode}, CamMC and PanShot translation metrics instead normalize the ground-truth and recovered trajectories independently by their respective maximum translation norms,
\begin{equation}
\bar{\mathbf{o}}_t
= \frac{\mathbf{o}_t^{\mathrm{rel}}}{\max_{1\le\tau\le T}\|\mathbf{o}_\tau^{\mathrm{rel}}\|_2+\epsilon_0},
\qquad
\bar{\widehat{\mathbf{o}}}_t
= \frac{\widehat{\mathbf{o}}_t^{\mathrm{rel}}}{\max_{1\le\tau\le T}\|\widehat{\mathbf{o}}_\tau^{\mathrm{rel}}\|_2+\epsilon_0},
\label{eq:eval-maxnorm}
\end{equation}
with $\epsilon_0=10^{-9}$.
The two scale treatments are not interchangeable: $\text{tr}\%$ uses $s^\star$, whereas $\text{CamMC}$ and $\text{TransErr}$ use Eq.~\eqref{eq:eval-maxnorm}; both rotation metrics are invariant to translation scaling.

\paragraph{Rotation Error ($\mathbf{\text{rot}^\circ}$).}
Used on driving evaluations.
The first-frame relative pose is the identity, so the mean excludes $t=1$:
\begin{equation}
\text{rot}^\circ
= \frac{1}{T-1}\sum_{t=2}^{T} \angle\!\big(R_t^{\mathrm{rel}},\,\widehat{R}_t^{\mathrm{rel}}\big).
\label{eq:eval-rot}
\end{equation}

\paragraph{Relative Translation Error ($\mathbf{\text{tr}\%}$).}
Used on driving evaluations.
After the scale alignment in Eq.~\eqref{eq:eval-scale}, the mean center error is normalized by the ground-truth endpoint displacement (chord length):
\begin{equation}
\text{tr}\%
= \frac{\frac{1}{T-1}\sum_{t=2}^{T}\|\mathbf{o}_t^{\mathrm{rel}}-\widetilde{\mathbf{o}}_t\|_2}{\|\mathbf{o}_T^{\mathrm{rel}}\|_2}\times 100\%.
\label{eq:eval-tr}
\end{equation}

\paragraph{Pairwise Joint Pose AUC ($\mathbf{\text{AUC}@\tau}$).}
Used on driving evaluations.
Let $\mathcal{P}=\{(i,j):1\le i<j\le T\}$ denote within-clip frame pairs, pooled over the evaluation set.
For each pair, form the relative motions $T_{i\to j}=(T_i^{\mathrm{rel}})^{-1}T_j^{\mathrm{rel}}$ and $\widehat{T}_{i\to j}=(\widehat{T}_i^{\mathrm{rel}})^{-1}\widehat{T}_j^{\mathrm{rel}}$, and define the unoriented translation-direction error
\begin{equation}
e_{ij}^{\mathrm{t}}
= \arccos\!\left(\operatorname{clamp}\!\left(\left|
\frac{\mathbf{o}_{i\to j}^\top\widehat{\mathbf{o}}_{i\to j}}{\|\mathbf{o}_{i\to j}\|_2\|\widehat{\mathbf{o}}_{i\to j}\|_2}
\right|,0,1\right)\right)\times\frac{180^\circ}{\pi},
\label{eq:eval-trans-dir}
\end{equation}
with $e_{ij}^{\mathrm{t}}$ set to a large sentinel when either translation vanishes.
The joint pair error is $e_{ij}=\max\bigl(\angle(R_{i\to j},\widehat{R}_{i\to j}),\,e_{ij}^{\mathrm{t}}\bigr)$.
Because $e_{ij}^{\mathrm{t}}$ depends only on direction, this metric does not apply $s^\star$.
With integer thresholds $\tau\in\{3,10\}$ (degrees), the discrete threshold average is
\begin{equation}
\text{AUC}@\tau
= \frac{1}{\tau}\sum_{k=1}^{\tau}
\frac{1}{|\mathcal{P}|}\sum_{(i,j)\in\mathcal{P}}\mathbb{I}\!\left(e_{ij}<k^\circ\right)\times 100\%.
\label{eq:eval-auc}
\end{equation}

\paragraph{Camera Motion Consistency ($\mathbf{\text{CamMC}}$).}
Let $P_t=[R_t^{\mathrm{rel}}\mid\bar{\mathbf{o}}_t]$ and $\widehat{P}_t=[\widehat{R}_t^{\mathrm{rel}}\mid\bar{\widehat{\mathbf{o}}}_t]$ be the $3\times 4$ relative pose matrices after Eq.~\eqref{eq:eval-maxnorm}.
Following UCPE, CamMC is the sum of per-frame Frobenius errors, then averaged over videos:
\begin{equation}
\text{CamMC}
= \sum_{t=1}^{T}\|P_t-\widehat{P}_t\|_F.
\label{eq:eval-cammc}
\end{equation}

\paragraph{PanShot Rotation Error ($\mathbf{\text{RotErr}}$).}
PanShot sums per-frame geodesic errors in radians, following its official evaluation code~\citep{zhang2026ucpecode}:
\begin{equation}
\text{RotErr}
= \sum_{t=1}^{T}\arccos\!\left(\operatorname{clamp}\!\left(\frac{\operatorname{tr}\!\left((R_t^{\mathrm{rel}})^\top\widehat{R}_t^{\mathrm{rel}}\right)-1}{2},-1,1\right)\right).
\label{eq:eval-panshot-rot}
\end{equation}

\paragraph{Translation Error ($\mathbf{\text{TransErr}}$).}
Used on PanShot evaluations, after the same per-trajectory max-norm as CamMC (not the driving scale $s^\star$):
\begin{equation}
\text{TransErr}
= \sum_{t=1}^{T}\|\bar{\mathbf{o}}_t-\bar{\widehat{\mathbf{o}}}_t\|_2.
\label{eq:eval-transerr}
\end{equation}

\section{Additional Experimental Analyses}
\label{app:additional-experiments}

\subsection{Architectural Design and Layer-Wise Modulation}
\label{app:architecture}

Table~\ref{tab:design-ablation} varies four design axes around the default \method configuration.
The Comp. rows vary only the adapter compression ratio $C$.
The Design rows fix $C=4$ and vary one design factor at a time: translation frequency, injection schedule, or integration style; check marks denote active mechanisms and blank cells denote inactive ones.

\paragraph{Adapter compression.}
The Wan2.2 backbone has width $3072$ with $24$ attention heads, and the ratio $C$ sets the parallel camera-attention branch to width $3072/C$ with $24/C$ heads while preserving the $128$-dimensional head size.
Thus, $C=2$, $4$, and $8$ use widths $1536$, $768$, and $384$, with $12$, $6$, and $3$ heads, respectively.
Although $C=2$ gives the strongest pose scores, the default $C=4$ retains similar CamMC with half the adapter capacity.

\paragraph{Translation frequencies.}
The Multi-$f$ column indicates whether metric translation uses multiple wavelength bands; the Single-$f$ variant keeps the same $128$-dimensional channel allocation and changes only the translation frequency.
It leaves directional AUC nearly unchanged but degrades tr\% and CamMC, showing that one wavelength does not provide reliable translation control across the evaluated trajectories.

\paragraph{Injection schedule.}
\begin{wrapfigure}{r}{0.48\textwidth}
    \vspace{-2mm}
    \centering
    \includegraphics[width=\linewidth]{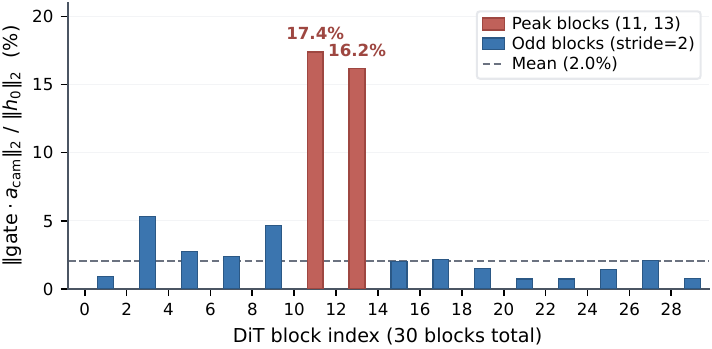}
    \vspace{-6mm}
    \caption{\textbf{Per-block camera modulation amplitude.}
    Relative camera-attention residual norm across $30$ DiT blocks; red marks the two largest blocks and blue marks alternating-block injection.}
    \label{fig:cam-residual}
    \vspace{-1mm}
\end{wrapfigure}

The Injection column compares three compute levels: Peak~11,13 uses two blocks, the default uses $15$ alternating blocks, and Every block uses all $30$.
Figure~\ref{fig:cam-residual} shows that camera modulation concentrates in the middle layers, with Blocks~$11$ and $13$ contributing over half of the residual amplitude, motivating the Peak~11,13 variant.
This variant preserves rotation accuracy and strong directional AUC while increasing CamMC only modestly from $2.32$ to $2.43$, and reduces camera-branch executions by $87\%$ relative to the default.
The default uses half as many camera-branch executions as Every block and has a similar CamMC ($2.32$ versus $2.27$).

\paragraph{Integration style.}
A check in CamSA denotes the dedicated parallel camera self-attention branch; a check in ReRoPE~\citep{li2026rerope} denotes carving low-frequency camera rotary channels into the backbone's native self-attention without that branch.
ReRoPE improves tr\% but weakens rotation and CamMC, so the parallel CamSA design provides the better overall balance.

\begin{table*}[t]
\centering
\caption{\textbf{Architecture and injection ablations on nuScenes.} Best in \textbf{bold}; second-best is underlined within each group.}
\label{tab:design-ablation}
\scriptsize
\setlength{\tabcolsep}{3.0pt}
\begin{tabular}{ll ccccc ccccc}
\toprule
& & \multicolumn{5}{c}{Configuration} & \multicolumn{5}{c}{Overall Metrics} \\
\cmidrule(lr){3-7} \cmidrule(lr){8-12}
& Variant & $C$ & Injection & CamSA & ReRoPE & Multi-$f$
& rot$^\circ$$\downarrow$ & tr\%$\,\downarrow$
& AUC@3$\uparrow$ & AUC@10$\uparrow$ & CamMC$\downarrow$ \\
\midrule
\multirow{3}{*}{Comp.}
& $C{=}2$ & $2$ & 15 blks & $\checkmark$ & & $\checkmark$
& $\mathbf{1.37}$ & $\mathbf{2.68}$ & $\mathbf{45.28}$ & $\mathbf{80.62}$ & $\mathbf{2.26}$ \\
& $C{=}4$ (ours) & $4$ & 15 blks & $\checkmark$ & & $\checkmark$
& $\mathbf{1.37}$ & $2.84$ & \underline{$42.11$} & \underline{$79.47$} & \underline{$2.32$} \\
& $C{=}8$ & $8$ & 15 blks & $\checkmark$ & & $\checkmark$
& \underline{$1.47$} & \underline{$2.83$} & $41.85$ & $79.11$ & $2.42$ \\
\midrule
\multirow{5}{*}{Design}
& Single $f$ & $4$ & 15 blks & $\checkmark$ & &
& $1.49$ & $3.58$ & $43.34$ & $79.61$ & $3.01$ \\
& Every block & $4$ & 30 blks & $\checkmark$ & & $\checkmark$
& \underline{$1.38$} & $\mathbf{2.65}$ & \underline{$43.60$} & \underline{$79.90$} & $\mathbf{2.27}$ \\
& Peak 11, 13 & $4$ & 2 blks & $\checkmark$ & & $\checkmark$
& $\mathbf{1.37}$ & $2.97$ & $\mathbf{45.00}$ & $\mathbf{80.27}$ & $2.43$ \\
& ReRoPE~\citep{li2026rerope} & $4$ & 15 blks & & $\checkmark$ & $\checkmark$
& $1.67$ & \underline{$2.67$} & $42.98$ & $79.71$ & $2.53$ \\
& Default (ours) & $4$ & 15 blks & $\checkmark$ & & $\checkmark$
& $\mathbf{1.37}$ & $2.84$ & $42.11$ & $79.47$ & \underline{$2.32$} \\
\bottomrule
\end{tabular}
\end{table*}

\subsection{Cost of Relaxing Per-Token Factorization}
\label{app:factorization-cost}

Strictly per-token encodings transform each token once, whereas \method expresses metric translation and disparity anchors in each query-camera frame and therefore constructs query-dependent key/value representations.
Table~\ref{tab:factorization-cost} measures the resulting cost at both the operator and full-model levels.

\paragraph{Evaluation protocol.}
The \textit{Single CamSA Block} columns time one complete camera self-attention (CamSA) block, including positional modulation, attention, and output projection, while fixing the block dimensions and input tokens across methods.
This isolated benchmark uses one H20 because it excludes the backbone and distributed model sharding.
The \textit{Training} and \textit{Latent Denoising} columns instead measure full training steps and the complete denoising loop on four H20 GPUs under the production four-way HSDP setup.

\paragraph{Matched comparison.}
For the full-model benchmark, GTA, UCPE, and \method inject CamSA into the same $15$ alternating transformer blocks (Blocks~$1,3,\ldots,29$), yielding $5.14$B parameters for each camera-conditioned model; the backbone-only baseline contains $5.00$B parameters.
This controls both the injection schedule and model size across the three camera encodings.

\begin{table*}[t]
\centering
\caption{\textbf{Computational cost of query-camera grouping.} Single-block results use one H20; full-model results use four H20 GPUs with a matched $15$-block injection schedule.}
\label{tab:factorization-cost}
\scriptsize
\setlength{\tabcolsep}{3.0pt}
\begin{tabular}{llcccccc}
\toprule
& & \multicolumn{2}{c}{Single CamSA Block} & \multicolumn{2}{c}{Training} & \multicolumn{2}{c}{Latent Denoising} \\
\cmidrule(lr){3-4} \cmidrule(lr){5-6} \cmidrule(lr){7-8}
Method & Factorization
& \shortstack{Latency\\(ms)$\downarrow$}
& \shortstack{Peak Mem.\\(GiB)$\downarrow$}
& \shortstack{Throughput\\(clips/s)$\uparrow$}
& \shortstack{Peak Mem.\\(GiB/GPU)$\downarrow$}
& \shortstack{Throughput\\(clips/s)$\uparrow$}
& \shortstack{Peak Mem.\\(GiB/GPU)$\downarrow$} \\
\midrule
Wan2.2 base & No CamSA & \multicolumn{2}{c}{N/A} & $5.30$ & $35.80$ & $0.197$ & $29.65$ \\
GTA~\citep{miyato2024gta} & Per-token & $4.97$ & $0.38$ & $4.89$ & $36.58$ & $0.188$ & $31.03$ \\
UCPE~\citep{zhang2026ucpe} & Per-token & $5.03$ & $0.40$ & $4.86$ & $36.58$ & $0.188$ & $31.03$ \\
\method & Query-grouped & $12.47$ & $1.21$ & $4.70$ & $37.57$ & $0.177$ & $31.13$ \\
\bottomrule
\end{tabular}
\end{table*}

\paragraph{Results.}
Query-camera grouping is more expensive within a single CamSA block, but the full backbone amortizes much of this local overhead.
Relative to UCPE, \method requires about $2.5\times$ the isolated-block latency and $3\times$ the peak memory, whereas full-model throughput decreases by $3.3\%$ for training and $5.9\%$ for denoising, with peak-memory increases of $2.7\%$ and $0.3\%$, respectively.
GTA and UCPE have nearly identical system costs under the matched schedule.

\paragraph{Source of overhead and future work.}
The extra local cost comes from constructing and storing separate transformed K/V tensors for each query-camera group rather than from additional attention dot products.
Processing query-camera groups in smaller batches can reduce peak memory, while a fused grouped-attention kernel could avoid expanded K/V tensors and reduce both memory and latency.

\subsection{Real-World Grounding with Retrieved History}
\label{app:history}

To test grounding beyond the current field of view, we encode calibrated same-location images with the video VAE and inject their patch tokens through two pathways.
History CA follows CityRAG~\citep{chou2026cityrag}: video tokens attend to retrieved-reference K/V through standard cross-attention.
History SA follows Seoul World Model~\citep{seo2026seoulworld}: retrieved reference tokens are concatenated with the current video tokens and jointly processed by self-attention across successive DiT blocks; only the generated video tokens are decoded.

\paragraph{Training signals.}
Current-image dropout randomly removes the entire clean conditioning prefix when usable history is available, so the model must predict from retrieved observations rather than copy the current clip; evaluation retains the full prefix.
Standard training retrieves history close to the ground-truth camera path, but left/right test commands often have only history with a larger relative pose, creating a train--test mismatch.
Lateral-conflict training simulates this case by substituting a spatially overlapping donor from a neighboring lane while leaving the target video and camera command unchanged, teaching the model to use off-trajectory history without weakening camera control.

\paragraph{Temporal tagging.}
Both pathways add $100$ to the matched history-frame time coordinates, placing retrieved tokens in a dedicated temporal region so the model can distinguish historical context from current video tokens.

\begin{table*}[t]
\centering
\caption{\textbf{Real-world grounding with retrieved history on nuScenes.} Best in \textbf{bold}; second-best is underlined.}
\label{tab:history}
\setlength{\tabcolsep}{3.5pt}
\small
\begin{tabular}{lccccccc}
\toprule
Method & rot$^\circ$$\downarrow$ & tr\%$\,\downarrow$
& AUC@3$\uparrow$ & AUC@10$\uparrow$ & CamMC$\downarrow$ & FID$\downarrow$ & FVD$\downarrow$ \\
\midrule
\method (no history) & $0.82$ & $2.21$
& $48.76$ & $83.24$ & $1.60$ & $19.82$ & $134.15$ \\
\method + History CA
& \underline{$0.78$} & \underline{$2.13$}
& $\mathbf{52.01}$ & \underline{$84.17$} & \underline{$1.52$} & \underline{$19.59$} & \underline{$130.27$} \\
\method + History SA
& $\mathbf{0.74}$ & $\mathbf{2.11}$
& \underline{$51.65$} & $\mathbf{84.22}$ & $\mathbf{1.48}$ & $\mathbf{18.21}$ & $\mathbf{104.41}$ \\
\bottomrule
\end{tabular}
\end{table*}

Both history variants provide modest pose-control gains, while History SA achieves the best CamMC, FID, and FVD in Table~\ref{tab:history}.
The VAE latents are optimized for pixel reconstruction rather than semantic alignment, so they provide no explicit high-level interface for retrieval fusion.
History SA repeatedly mixes retrieved and future tokens across successive DiT blocks, whereas History CA exposes the retrieved latents only as cross-attention K/V; this difference may explain History SA's stronger visual anchoring.
Figure~\ref{fig:history-anchor} illustrates this pattern: History CA follows the no-history layout at late timesteps, while History SA preserves retrieval-supported geometry in both selected scenes.

\begin{figure*}[!ht]
  \centering
  \newcommand{\historyheader}{%
    \noindent\hspace*{0.045\textwidth}%
    \parbox{0.95\textwidth}{\centering\footnotesize
      \makebox[0.2\linewidth][c]{Ground truth}%
      \makebox[0.2\linewidth][c]{Retrieved history}%
      \makebox[0.2\linewidth][c]{No history}%
      \makebox[0.2\linewidth][c]{\textcolor[RGB]{238,119,51}{History CA}}%
      \makebox[0.2\linewidth][c]{\textcolor[RGB]{0,119,187}{History SA}}%
    }\par\vspace{1pt}%
  }
  \newcommand{\historyrowlabel}[1]{%
    \noindent\makebox[0.045\textwidth][r]{\scriptsize $#1$\hspace{0.35em}}%
  }
  \textbf{(a) Tree-lined corridor}\par\vspace{2pt}
  \historyheader
  \historyrowlabel{t=0}\includegraphics[width=0.95\textwidth]{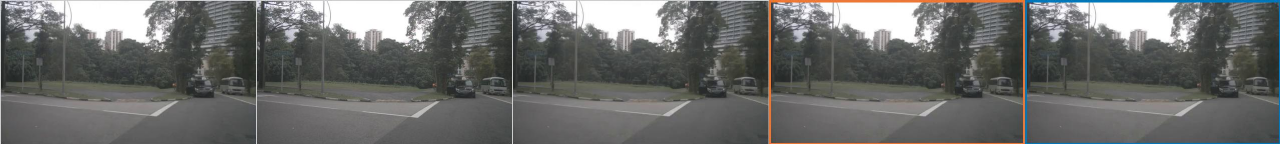}\par\vspace{1pt}
  \historyrowlabel{t=32}\includegraphics[width=0.95\textwidth]{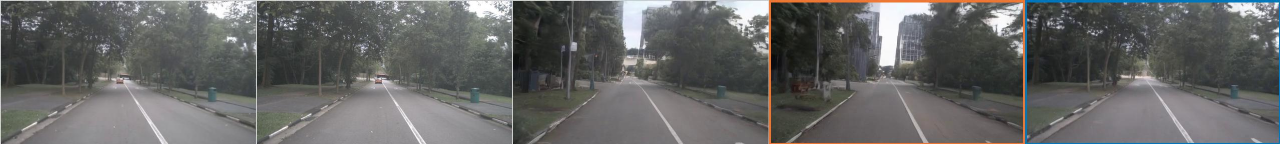}\par\vspace{1pt}
  \historyrowlabel{t=48}\includegraphics[width=0.95\textwidth]{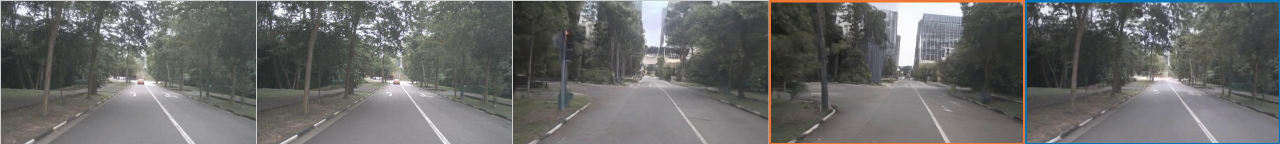}\par\vspace{1pt}
  \vspace{5pt}
  \textbf{(b) Cross-weather street corner}\par\vspace{2pt}
  \historyheader
  \historyrowlabel{t=0}\includegraphics[width=0.95\textwidth]{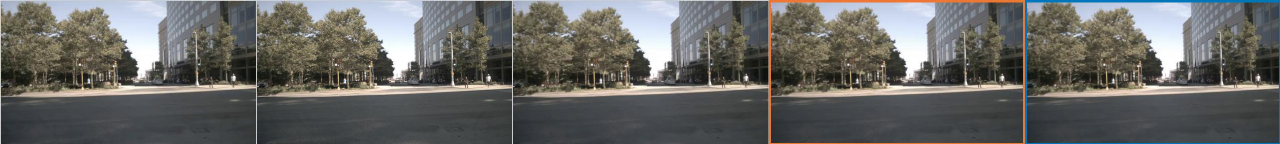}\par\vspace{1pt}
  \historyrowlabel{t=32}\includegraphics[width=0.95\textwidth]{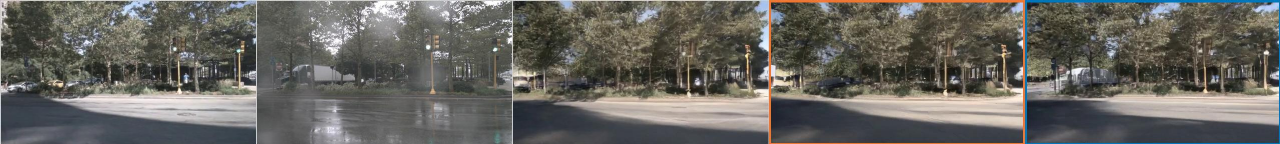}\par\vspace{1pt}
  \historyrowlabel{t=48}\includegraphics[width=0.95\textwidth]{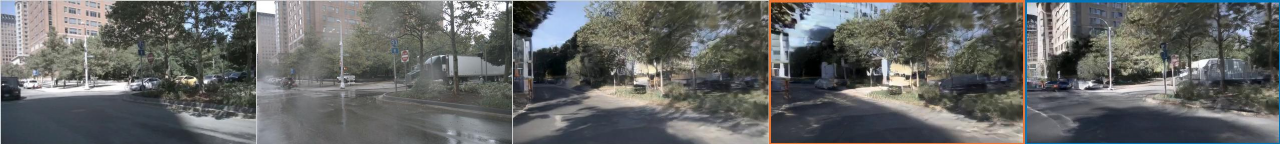}\par\vspace{1pt}
  \caption{\textbf{Qualitative comparison of history injection.} Across the same scenes and timesteps, History SA preserves retrieval-supported late-frame geometry, whereas History CA follows the no-history layout.}
  \label{fig:history-anchor}
\end{figure*}
\FloatBarrier

\section{Additional Qualitative Results}
\label{app:qualitative}

\subsection{Matched-Backbone Comparison with Prior Camera Encodings}
\label{app:qual-sota}

Figure~\ref{fig:qual-sota} presents three fixed clips separately for readability; Table~\ref{tab:main-video} reports the aggregate evaluation.
In these clips, \method preserves more coherent late-frame geometry.

\begin{figure*}[p]
  \centering
  \includegraphics[width=\textwidth]{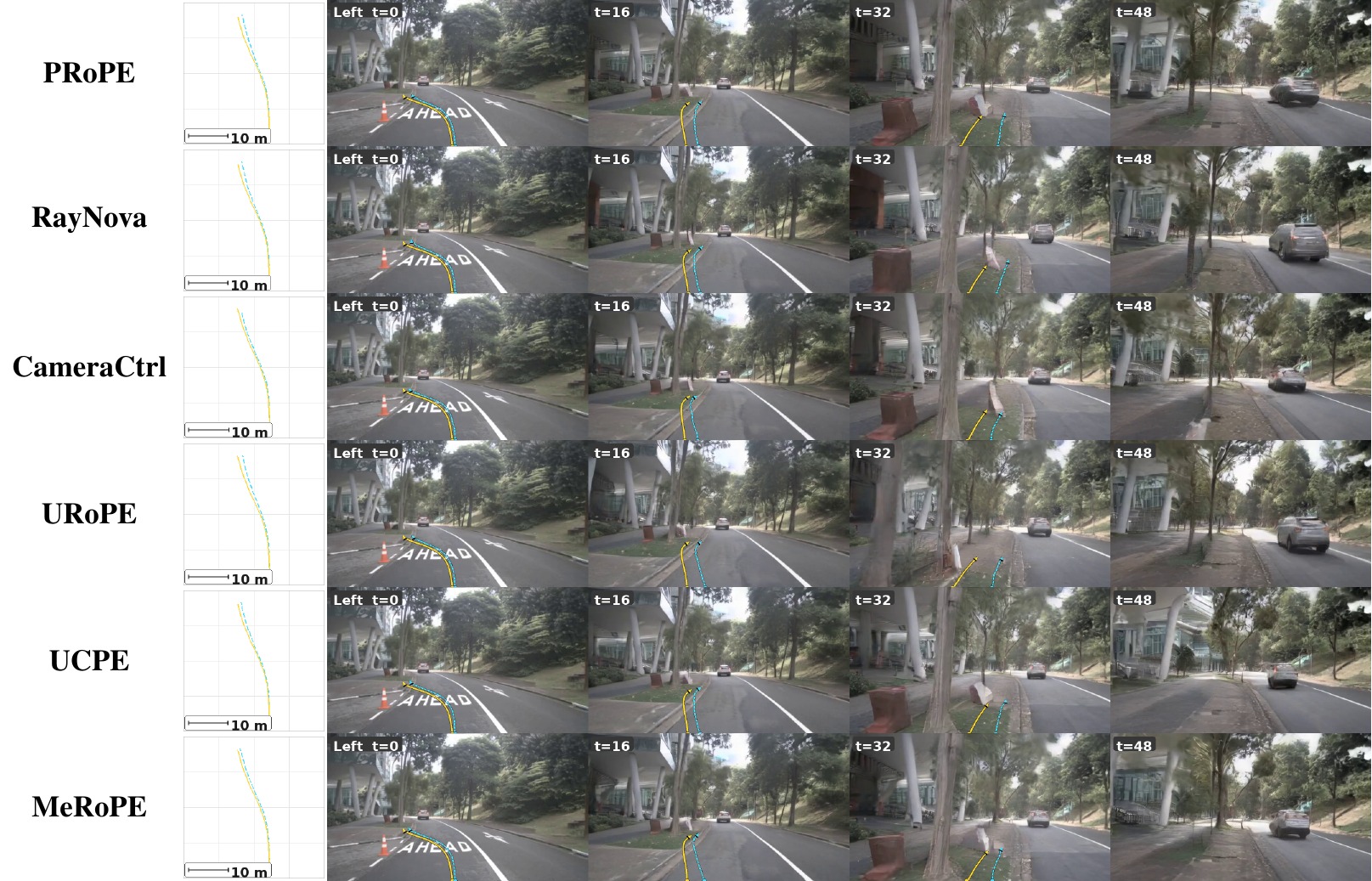}
  \caption{\textbf{Matched-backbone qualitative comparison (1/3).} Each row shows commanded and VGGT-$\Omega$-recovered paths with frames at $t\in\{0,16,32,48\}$.}
  \label{fig:qual-sota}
\end{figure*}
\begin{figure*}[p]
  \ContinuedFloat
  \centering
  \includegraphics[width=\textwidth]{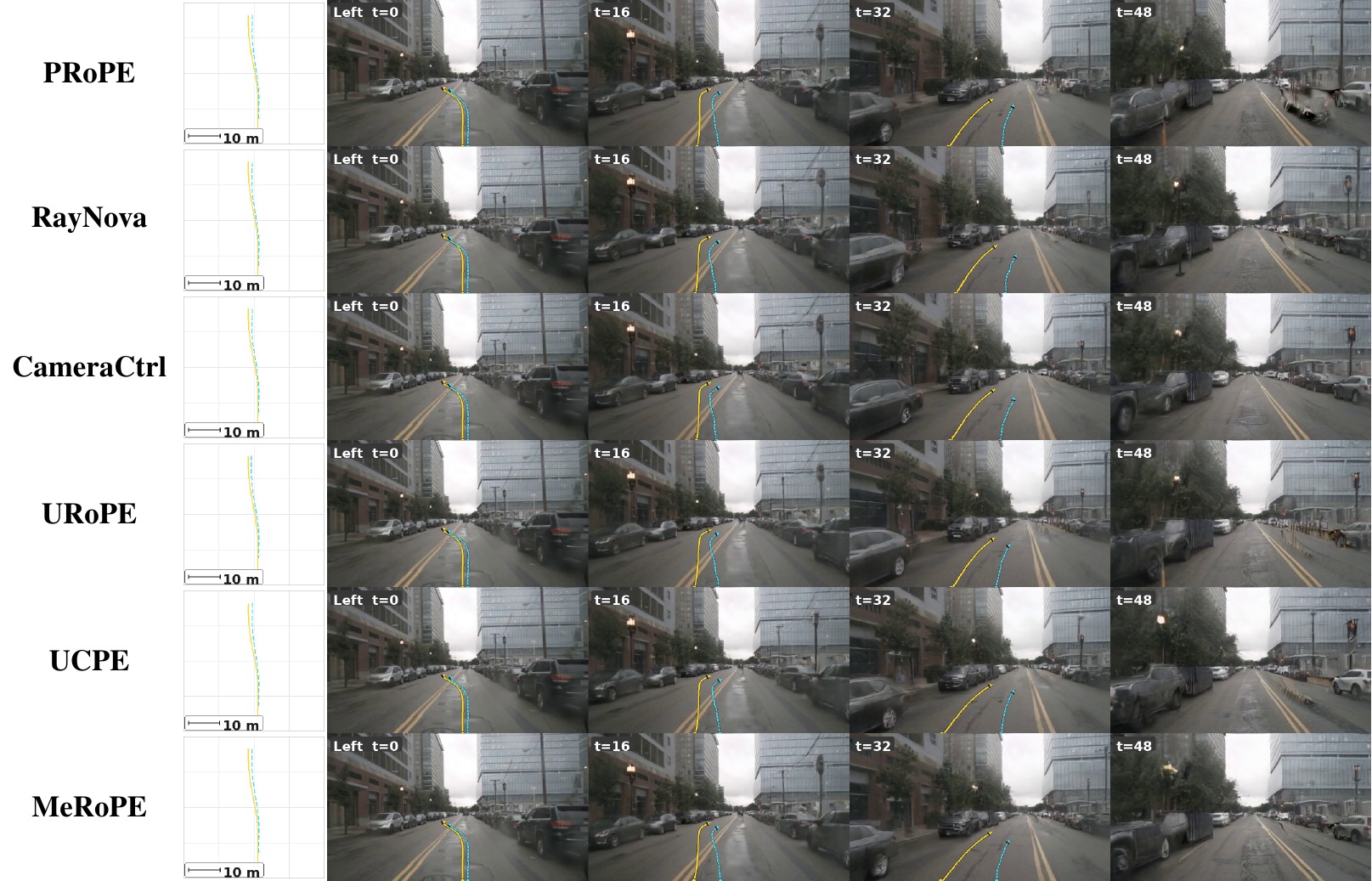}
  \caption{\textbf{Matched-backbone qualitative comparison (2/3).} The same protocol is shown for the second scene.}
  \label{fig:qual-sota-2}
\end{figure*}
\begin{figure*}[t]
  \ContinuedFloat
  \centering
  \includegraphics[width=\textwidth]{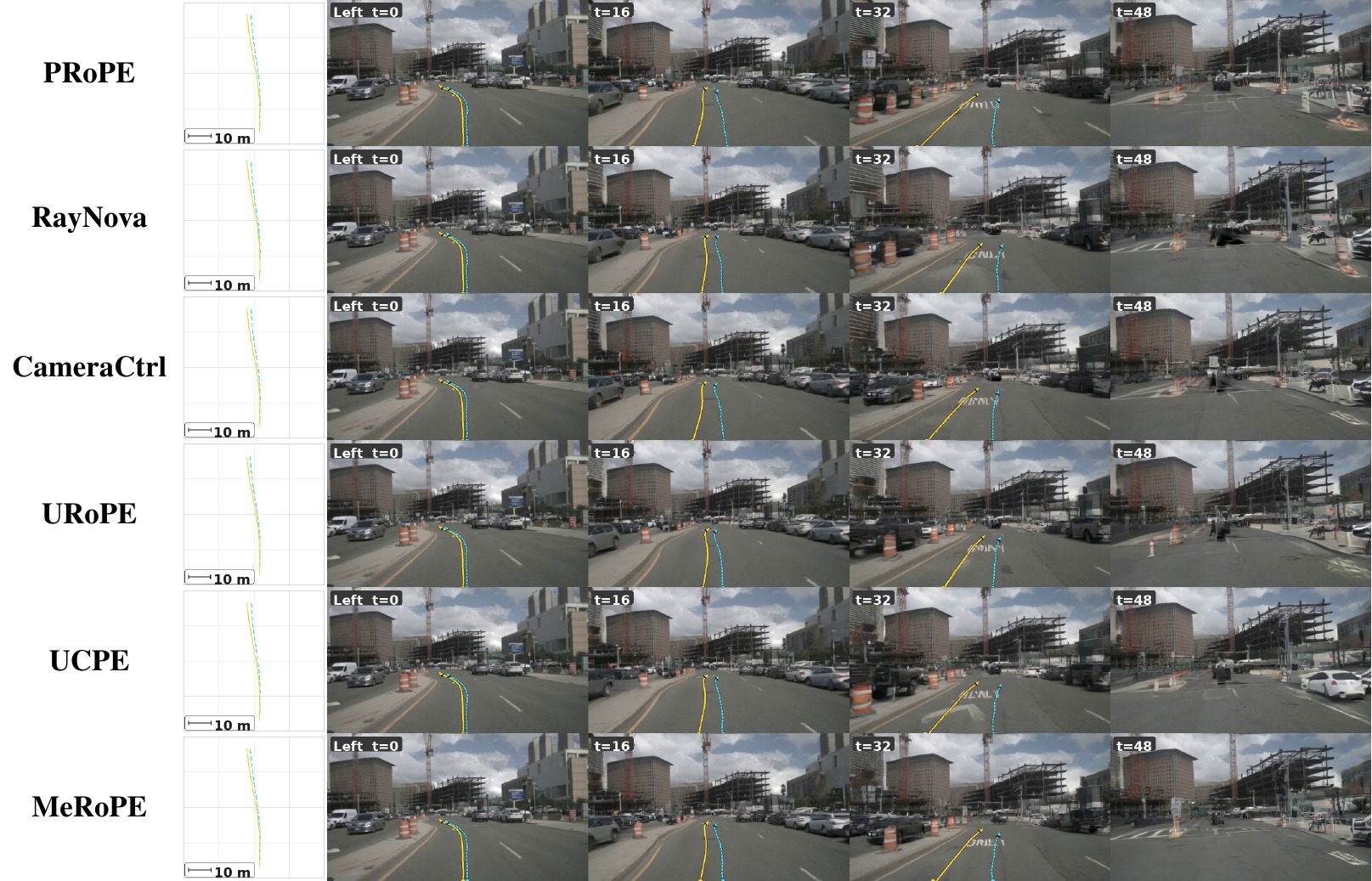}
  \caption{\textbf{Matched-backbone qualitative comparison (3/3).} The same protocol is shown for the third scene.}
  \label{fig:qual-sota-3}
\end{figure*}
\FloatBarrier

\subsection{Additional Camera-Pose-Controlled Generations}
\label{app:qual-pose-more}

Figure~\ref{fig:qual-pose-more} extends the main qualitative comparison to three additional nuScenes scenes.

\begin{figure*}[!ht]
  \centering
  \includegraphics[width=\textwidth]{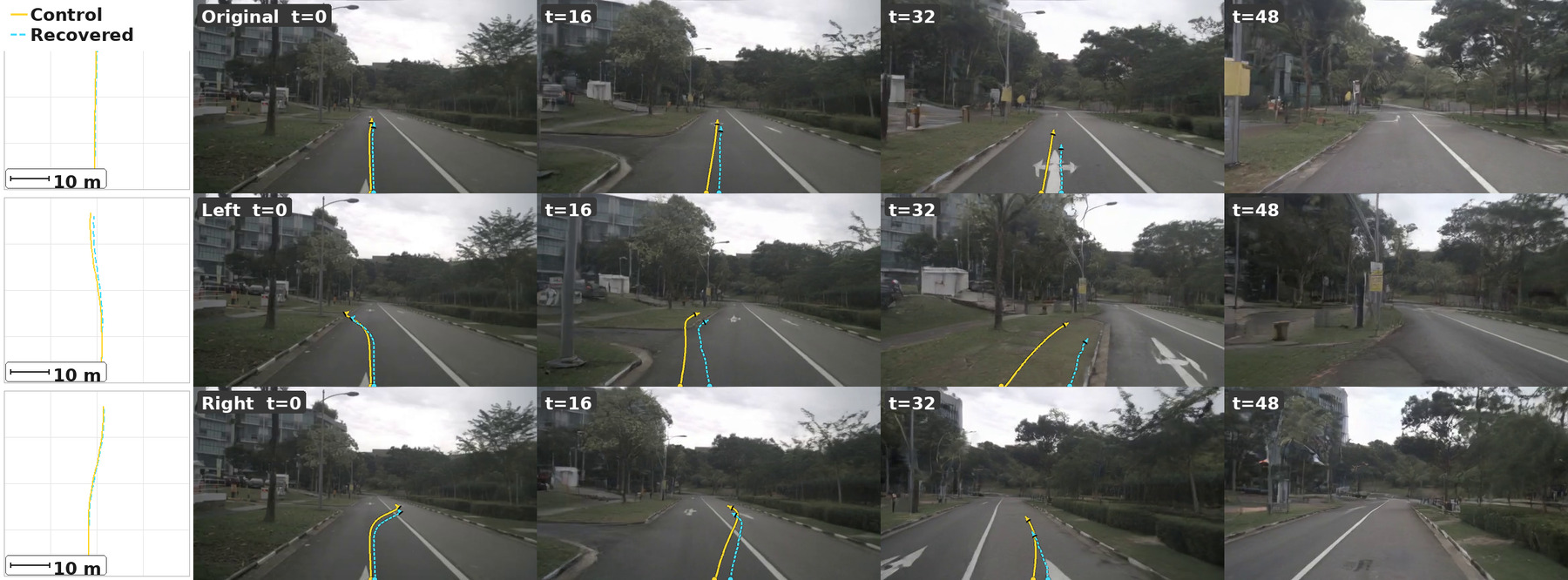}\\[2pt]
  \includegraphics[width=\textwidth]{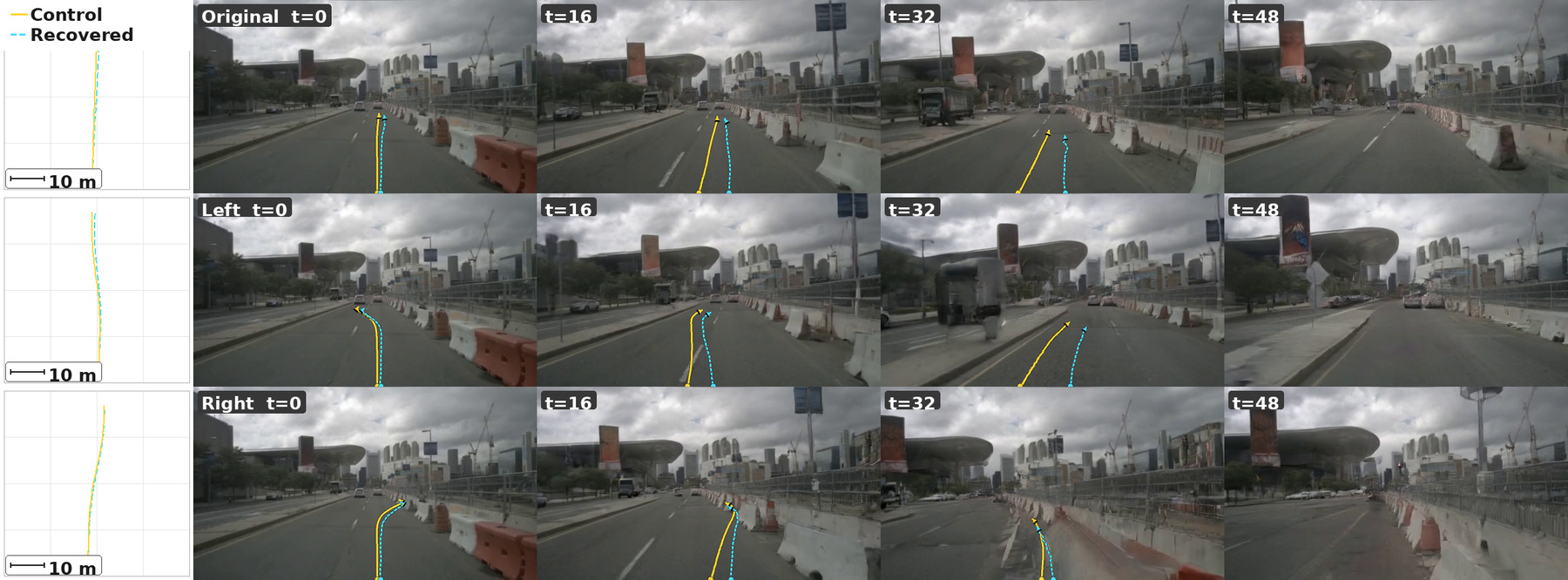}\\[2pt]
  \includegraphics[width=\textwidth]{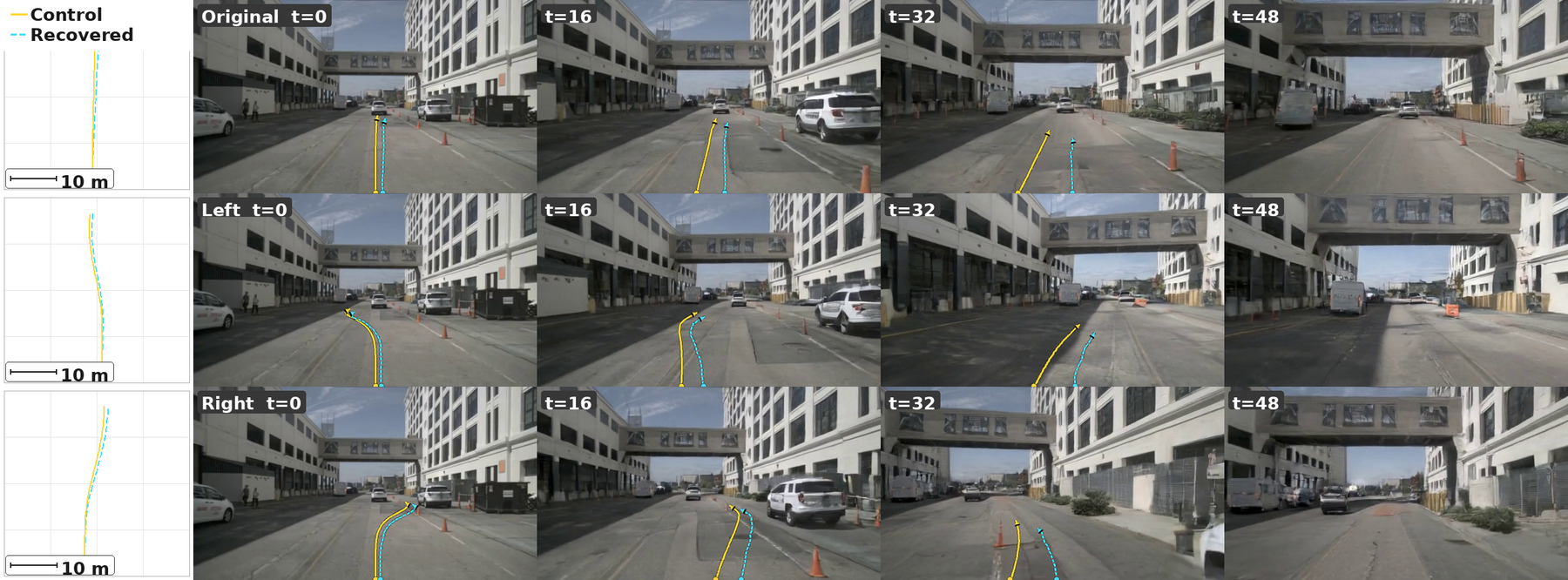}
  \caption{\textbf{Additional nuScenes camera-control results.} For each scene, the left plot compares commanded (yellow solid) and VGGT-$\Omega$-recovered (cyan dashed) paths; the remaining columns show frames at $t\in\{0,16,32,48\}$.}
  \label{fig:qual-pose-more}
\end{figure*}

\end{document}